%% file: ex_article.tex
\documentclass[hidelinks,onefignum,onetabnum]{siamart250211}

\input{ex_shared}

\ifpdf
\hypersetup{
  pdftitle={FROM EQUATIONS TO INSIGHTS: UNRAVELING SYMBOLIC
STRUCTURES IN PDES WITH LLMS},
  pdfauthor={Bhatnagar, Liang, Patel, and Yang}
}
\fi

\begin{document}

\maketitle

\begin{abstract}
Motivated by the remarkable success of artificial intelligence (AI) across diverse fields, the application of AI to solve scientific problems, often formulated as partial differential equations (PDEs), has garnered increasing attention. While most existing research concentrates on theoretical properties (such as well-posedness, regularity, and continuity) of the solutions, alongside direct AI-driven methods for solving PDEs, the challenge of uncovering symbolic relationships within these equations remains largely unexplored. In this paper, we propose leveraging large language models (LLMs) to learn such symbolic relationships. Our results demonstrate that LLMs can effectively predict the operators involved in PDE solutions by utilizing the symbolic information in the PDEs both theoretically and numerically. Furthermore, we show that discovering these symbolic relationships can substantially improve both the efficiency and accuracy of symbolic machine learning for finding analytical approximation of PDE solutions, delivering a fully interpretable solution pipeline. This work opens new avenues for understanding the symbolic structure of scientific problems and advancing their solution processes.
\end{abstract}



\section{Introduction}\label{sec:introduction}

Partial differential equations (PDEs) are the mathematical language of the natural sciences: they encode how fields evolve and interact by linking an unknown function on space–time to its derivatives. From turbulent flows and elastic structures to reaction–diffusion in biology and quantum systems, many grand scientific challenges are governed by PDE models \cite{evans2022partial}. 

Mathematically, a strong-form PDE constrains an unknown function $u:\Omega\to\mathbb{R}$ on a domain $\Omega\subseteq\mathbb{R}^d$ by relating $u$ to its partial derivatives together with boundary data
\begin{equation} \label{eq-pde} 
Du = f,\quad \text{in } \Omega, \quad Bu = g,\quad \text{on }\partial \Omega.
\end{equation}
Here $D$ is a differential operator of order $k$ that may depend on $x$, $u$, and derivatives of $u$ up to order $k$; a linear prototype is $Du=\sum_{|\alpha|\le k} a_\alpha(x)\,D^\alpha u$ with multi-index $\alpha=(\alpha_1,\ldots,\alpha_d)$ and $D^\alpha u:=\partial^{|\alpha|}u/\partial x_1^{\alpha_1}\cdots\partial x_d^{\alpha_d}$, while $f:\Omega\to\mathbb{R}$ is a given source term and $g:\partial \Omega \to \mathbb{R}$ is given. The boundary operator $B$ prescribes data on $\partial\Omega$. Some commonly used boundary conditions include $Bu=\gamma_0 u=u|_{\partial\Omega}=g$ (Dirichlet), $Bu=\gamma_1 u=\partial_{\boldsymbol n}u=g$ with $\boldsymbol n$ the outward unit normal (Neumann), or $Bu=\gamma_1 u+\kappa\,\gamma_0 u=g$ for $\kappa\ge0$ (Robin) where $\gamma_0:H^1(\Omega)\to H^{1/2}(\partial\Omega)$ and $\gamma_1:H^1(\Omega)\to H^{-1/2}(\partial\Omega)$ denote continuous trace maps. For time-dependent problems on $Q=(0,T)\times\Omega$, one augments the system with the initial condition $u(0,\cdot)=u_0$ in $\Omega$. Classical examples include the parabolic and hyperbolic models
\[
u_t-\nabla\!\cdot(A\nabla u)=f, \qquad u_{tt}-\nabla\!\cdot(A\nabla u)=f,
\]
where $A$ denotes the diffusion/conductivity tensor. When classical derivatives are unavailable, the statement $Du=f$ and $Bu=g$ is interpreted in a weak (variational) sense on Sobolev spaces, where existence and uniqueness follow under standard boundedness and coercivity assumptions (e.g., by the Lax–Milgram theorem). And the principal symbol of $D$ then classifies the PDE as elliptic, parabolic, or hyperbolic.

PDEs often arise in high-dimensional settings (i.e., $d$ is large); for example, well-known instances such as the Poisson equation \cite{yu2018deep}, the linear conservation law \cite{chen2021deep}, and the nonlinear Schrödinger equation \cite{han2020solving} naturally involve multiple spatial and temporal variables. While low-dimensional PDEs can be solved effectively via traditional mesh-dependent methods such as finite difference methods \cite{leveque2007finite}, finite elements methods \cite{hughes2003finite}, finite volume methods \cite{toro2013riemann}, and spectral methods \cite{boyd2001chebyshev}, developing efficient and accurate algorithmic frameworks for computing numerical
solutions to high-dimensional PDEs remains an important and challenging topic due to the curse of dimensionality (i.e., the computational costs grow exponentially with respect to the dimensionality) \cite{weinan2021algorithms}. 

Before solving a PDE, one must understand certain fundamental properties of the solution $u$ that follow from the PDE itself and from the data, namely, $f$ and $g$. These considerations align with the classical theory of PDEs \cite{evans2022partial}, which focuses on establishing well-posedness (i.e., the existence and uniqueness of solutions) and their regularity and continuity. However, the underlying symbolic relationship among $u$, $f$ and $g$ often remains underexplored. To address this gap, we pose the following question: \emph{Assuming that the PDE \eqref{eq-pde} admits an analytical solution $u$, and given the operators in $f$ and $g$, how are the operators in $u$ related to those in $f$ and $g$?}

Understanding this symbolic connection is crucial for deriving fully interpretable, closed-form solutions, a feature lacking in most existing solution methods, particularly in the high-dimensional setting. As an illustrative example, we consider solving the PDE \eqref{eq-pde} via the least square method \cite{dissanayake1994neural,lagaris1998artificial,sirignano2018dgm}, which defines a straightforward loss to characterize the error of the estimated solution by
\begin{equation}\label{eq-least-squared}
    L(u) = \| D u - f \|^2_{L^2(\Omega)} + \lambda \| B u - g \|^2_{L^2(\partial \Omega)},
\end{equation}
where \( \lambda > 0 \) balances the influence of boundary conditions. The goal is to find a function $u^*\in L^2(\Omega)$ so that the least square loss is minimized, i.e., 
\begin{equation}
    \label{eq-ls-min}
    u^* = \mathrm{argmin}\left\{ L(u)\;:\; u\in L^2(\Omega)\right\}.
\end{equation}
Because the search space $L^2(\Omega)$ is an infinite dimensional space, finding the solution $u^*$ is extremely challenging. Fortunately, the search space can be significantly simplified if one knows the operators appearing in the solution in advance, making the search of the optimal solution much easier and effective. This can be accomplished by adopting state-of-the-art symbolic machine learning approaches, including the efficient finite expression method (FEX) of \cite{liang2022finite}.

To uncover the symbolic relationship, we leverage large language models (LLMs) as predictive tools. The process begins with the generation of a comprehensive, structured dataset composed of symbolic expressions derived from a diverse range of PDE types, including elliptic, parabolic, and hyperbolic equations, along with their associated boundary and initial conditions. These symbolic expressions are represented as computational trees (see Fig.~\ref{fig:comp_expr_tree}), where nodes are populated with unary (e.g., sin, cos, exp) and binary operators (e.g., +, -, *, /). This tree structure captures the hierarchical nature of mathematical expressions and ensures a systematic representation of the PDE solutions. 

\begin{figure*}[htb!]
    \centering
    \includegraphics[width=0.98\linewidth]{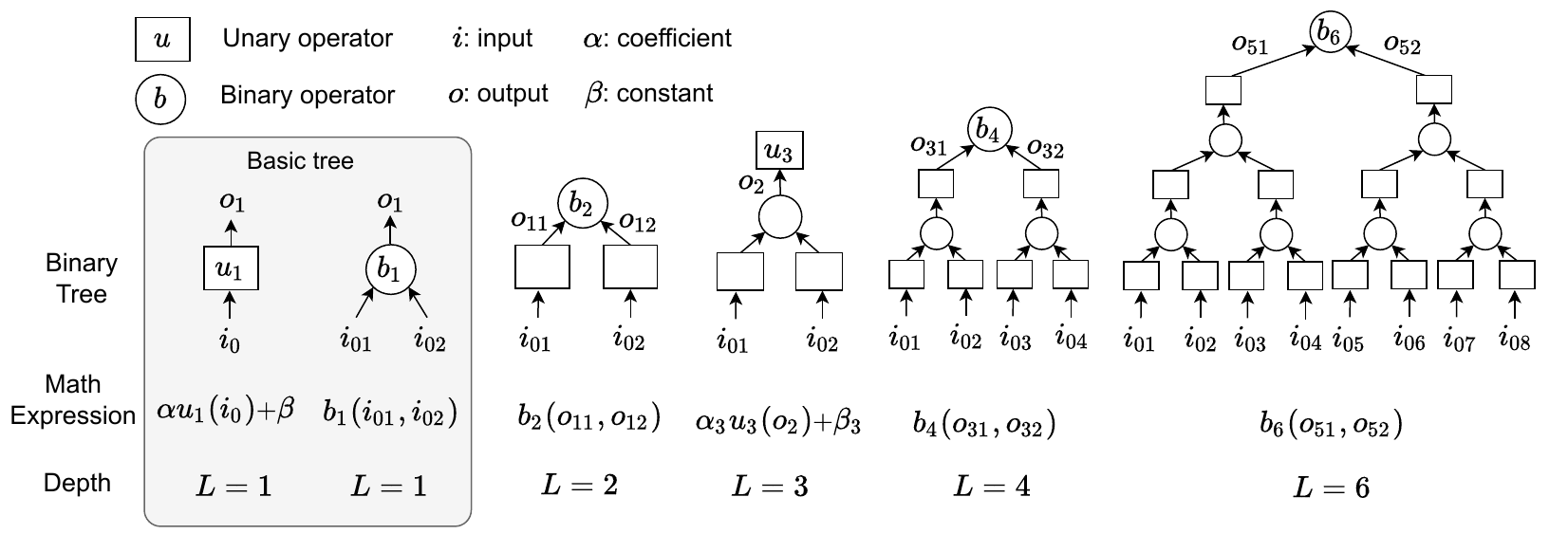}
    \caption{Computational expression tree~\cite{liang2022finite}. }
    \label{fig:comp_expr_tree}
\end{figure*}

To prepare the data for LLM processing, the symbolic expressions are tokenized and converted into postfix notation (Reverse Polish Notation). This format is chosen because it aligns well with the sequential input requirements of LLMs and eliminates the need for parentheses, simplifying the parsing process. More importantly, it reflects the relative positional relationship of operators in the expression. The tokenized dataset is then used to fine-tune foundational LLMs. During fine-tuning, the models are trained to predict the sequence of operators that constitute the symbolic solution to a new PDE. The training objective is to maximize the accuracy of operator sequence prediction, ensuring that the model can reliably identify the minimal set of operators required to represent the true solution.

Once fine-tuned, the LLM acts as a predictive tool to generate operator sets for new PDE problems. The predicted operator sets are highly refined, containing only the operators necessary to construct the solution, thereby eliminating redundant or irrelevant operators. This refinement is crucial for the subsequent application of the finite-expression method (FEX), a reinforcement learning–based symbolic regression technique that searches for the optimal combination of operators and constants to approximate the PDE solution. By providing FEX with a reduced and targeted operator set, the combinatorial search space is significantly narrowed. This reduction not only enhances the accuracy of the solution by focusing on relevant operators, but also dramatically improves computational efficiency by reducing the number of potential combinations that need to be evaluated. While we adopt FEX in this work, it is not the only viable approach. If the symbolic relationship between a PDE and its solution can be effectively learned, then a wide range of symbolic regression methods \cite{landajuela2022unified}, such as evolutionary algorithms \cite{schmidt2009distilling}, can be directly applied using the learned operator set.

In summary, the integration of LLMs into this workflow transforms the process of symbolic PDE solving. By generating structured datasets, fine-tuning models to predict operator sequences, and leveraging these predictions to guide PDE solvers, we achieve a more accurate and computationally efficient approach to discovering symbolic solutions for complex PDEs. This methodology bridges the gap between machine learning and symbolic mathematics, offering great potential for scientific and engineering applications. Particularly, our contributions are summarized as follows:
\begin{enumerate}
\item \textbf{PDE symbolic relations.} We formally pose the problem of discovering the symbolic relationship between the solution of a PDE and the associated problem data, including the boundary conditions and source functions. Establishing this connection can provide deeper insights into the underlying structure of PDEs, enabling more precise analytical solutions and facilitating the development of efficient numerical methods which provides high-quality analytical solutions. By highlighting the potential impact of such relationships, our work lays the foundation for further studies, aiming to bridge the gap between problem formulation and solution representation.

\item \textbf{LLM-guided symbolic PDE solvers.} We fine-tune LLMs to predict the operator set that captures the solution’s structure and couple these predictions with the FEX pipeline to search over symbolic programs that produce fully interpretable surrogates $u_\theta$. Experiments show that, while LLMs are not direct PDE solvers, their operator- and geometry-aware priors sharply prune the search over expressions, improving FEX’s sample efficiency and stability and yielding higher-quality closed-form or near–closed-form solutions across PDE families. The same operator-selection front end is modular and can be integrated with other symbolic PDE solvers where operator choice is the critical bottleneck.

\item \textbf{Theory for policy-gradient–based symbolic learning.} We provide convergence guarantees for a stochastic projected policy-gradient method that maximizes the physics-informed return over symbolic programs. In particular, under standard conditions, we prove that projected updates converge to first-order stationary policies and establish a sample complexity of $\mathcal{O}(\epsilon^{-4})$ to reach an $\epsilon$-stationary point in expectation, thereby placing the LLM+FEX pipeline on firm theoretical footing among state-of-the-art alternatives.
\end{enumerate}

\subsection{Related work}

Artificial intelligence (AI) for Science takes advantage of modern advanced machine learning (ML) methods to facilitate, accelerate and enhance scientific discovery \cite{stevens2020ai}. Its importance lies in the ability to handle massive datasets, uncover hidden patterns, and simulate complex phenomena. Inspired by the remarkable success of AI across various domains, deep neural network (DNN)-based approaches  \cite{lecun2015deep,goodfellow2016deep} have become increasingly popular for solving (high-dimensional) PDEs. These methods offer certain appealing advantages over traditional numerical techniques, making them a compelling alternative for tackling complex problems. First, DNN-based approaches empirically mitigate the curse of dimensionality that plagues conventional discretization approaches, leading to more efficiency in handling high-dimensional PDEs \cite{han2018solving}. Second, DNNs, acting as mesh-free universal function approximators \cite{cybenko1989approximation,hornik1989multilayer,goodfellow2016deep}, can be trained to learn highly complex PDE solutions without the need of prespecifying basis functions or meshes \cite{berg2018unified,raissi2019physics}. Hence, they are highly suitable to approximate complex solution spaces that challenges the traditional mesh-based approaches. Third, the solutions represented by DNNs can be evaluated efficiently once trained, which turns out to be a crucial property for tasks, including uncertainty quantification and real-time optimal control, require calling PDE solutions repeatedly \cite{zhu2019physics}. Last but not least, continuous progress in designing comprehensive neural network architectures, such as Fourier neural operators and advanced physics-informed neural network frameworks, alongside theoretical advancements in understanding their approximation capabilities, are making these methods increasingly applicable \cite{lu2019deeponet,li2020fourier}.

However, several critical factors often hinder deep neural networks (DNNs) from achieving highly accurate solutions, even for relatively simple problems. These include the need for large, diverse, and high-quality training data \cite{szegedy2013intriguing,buda2018systematic}, sensitivity to hyperparameter selection \cite{bergstra2012random}, the challenge of optimizing highly nonconvex objectives \cite{dauphin2014identifying,choromanska2015loss}, issues like vanishing and exploding gradients \cite{glorot2010understanding}, and poor generalization capabilities \cite{zhang2021understanding}. Additionally, solutions produced by DNNs often lack interpretability, preventing users from leveraging insights for future decision-making \cite{rudin2019stop}. To overcome these challenges and achieve both highly accurate and interpretable solutions for PDEs, \cite{liang2022finite} recently introduced the finite expression method (FEX). This approach leverages recent advances in approximation theory, which demonstrate that high-dimensional functions can be effectively approximated using functions composed of simple operators \cite{shen2021deep,shen2021neural,zhang2022deep}. This methodology seeks approximate PDE solutions within the space of functions composed of finitely many analytic expressions, effectively avoiding the curse of dimensionality. In FEX, the mathematical expression representing the PDE solution is modeled as a binary tree, where each node holds either a unary or binary operator chosen from pre-specfied operator sets. The goal is to find the optimal sequence of operators and coupling parameters that minimize the objective function associated with the PDE. This leads to a combinatorial optimization (CO) problem, which is then solved by a deep reinforcement learning (RL) method with the policy gradient optimizer. The FEX outperforms certain existing numerical PDE solvers in terms of accuracy, interpretability and memory efficiency, making it a promising approach for solving high-dimensional PDEs and other complex problems; see e.g., \cite{song2023finite,song2024finite,hardwick2024solving,liang2024stochastic}. 

\section{Theoretical Insight for Poisson Equations}\label{sec:theory}
While providing a complete answer to the question raised in the Introduction remains a significant challenge, in this section we take an initial step toward addressing it by offering theoretical insights that illuminate key aspects of the problem and motivate directions for future research. 

Specifically, we focus on the classical Poisson equation with Dirichlet boundary condition
\begin{equation}
    \label{eq-poisson}
    -\Delta u = f \quad \text{in } \Omega, \qquad u = g \quad \text{on } \partial \Omega,
\end{equation}
where $\Omega\subseteq \mathbb{R}^d$ is a compact domain. By analyzing this simple class of elliptic PDEs under, we aim to shed light on the structural properties and solution behavior relevant to the broader question, thereby laying the groundwork for more comprehensive theoretical development in future work.

\begin{theorem}
    \label{thm-poisson}
    Let $\Omega \subset \mathbb{R}^d$ be a compact domain with boundary $\partial \Omega$ such that the distance function $\mathcal{D}(x) := \mathrm{dist}(x, \partial \Omega)$ admits an explicit analytical expression. Suppose the source term $f \in C^{0,\alpha}(\Omega)$ for some $\alpha \in (0,1]$ (i.e., $f$ is H\"older continuous), and the boundary condition $g$ is the restriction to $\partial \Omega$ of a function in $C^2(\overline{\Omega})$. Let $u$ be the analytical solution to the Possion equation \eqref{eq-poisson}. Then for any $\delta \in (0,1)$, there exists an approximating function $\tilde{u}: \mathbb{R}^d \to \mathbb{R}$ such that:
\begin{enumerate}
    \item $\|u - \tilde{u}\|_{L^2(\Omega)} \leq \delta$,
    \item $\tilde{u}$ is constructed using only the elementary arithmetic operators $+,\;\times,\; (\cdot)^2$ and the operators appearing in the evaluation of $\mathcal{D}(\cdot)$, $f(\cdot)$, and $g(\cdot)$,
    \item The total number of operations needed to evaluate $\tilde{u}(x)$ at any $x \in \mathbb{R}^d$ satisfies:
    \[
    \texttt{ops}(\tilde{u}) = \texttt{poly}(\delta^{-1}) \cdot \left(\texttt{ops}(f) + \texttt{ops}(g)\right),
    \]
    where $\texttt{poly}(\delta^{-1})$ denotes a quantity polynomial in $\delta^{-1}$.
\end{enumerate}
\end{theorem}
\begin{proof}
Let $W_t,\; t\geq 0$ be a $d$-dimensional Brownian motion such that $W_t$ is $F_t$-measurable with probability measure $P$. We define the stochastic process
    \[
        X_t = X_0 + W_t, \quad t\geq 0,
    \]
    and denote by $P_x$ the probability measure conditioned on $X_0 = x$ for any $x\in \mathbb{R}^d$, and the expectation with respect to $P_x$ is denoted as $\mathbb{E}_x[\cdot]$. Accordingly, we define the discrete process $\bar{X}_k,\; k\geq 0$ as follows:
    \[
        \bar{X}_k = \bar{X}_{k-1} + Y_k \mathcal{D}(\bar{X}_{k-1}),\quad k\geq 1, \quad X_0 = x,
    \]
    where $Y_k, \; k\geq 1$ is indpendently, indentically and uniformly distribution on the unit sphere according to $X_{\tau(B(0,1))}$ with respect to a probability distribution $P_0$. Obviously, $\bar{X}_k$ depends on $x, Y_1, \dots, Y_k$, hence to emphasize this dependency, we also dentoe  
    \[
    \bar{X}_k:= \bar{X}_k(x, Y_1,\dots, Y_k),\quad k\geq 1.
    \]
    For any open, nonempty set $\mathcal{X}\subseteq \mathbb{R}^d$, we define the first exit time of the process $X_t$ starting at $x\in \mathcal{X}$ from $\mathcal{X}$ by 
    \[
        \tau(\mathcal{X}):= \inf\left\{t\geq 0\;:\; X_t \notin \mathcal{X}\right\}.
    \]
    Then, from \cite{kyprianou2018unbiased}, we see that for $x\in \Omega $, the solution $u(x)$ for the Poisson equation admits the following expression:
    \begin{equation}
        \label{eq-sol-poisson}
        \begin{aligned}
            &\;u(x) =  \mathbb{E}_x\left[g(X_{\tau(\Omega)})\right] + \frac{1}{2}\mathbb{E}_x\left[\sum_{k\geq 1}\mathcal{D}(\bar{X}_{k})^2E_0\right],
        \end{aligned}
    \end{equation}
    where 
    \[
        E_0 := \mathbb{E}_0\left[\int_0^{\tau(B(0,1))}f\left(\bar{X}_{k-1}+D(\bar{X}_{k-1})\cdot t\right)dt\right].
    \]

    We first consider approximating the first term in \eqref{eq-sol-poisson} using the Monte-Carlo sampling technique. To this end, we define the function $g^{(M, \{K_i\})}:\mathbb{R}^d\to \mathbb{R}$ as 
    \[
        g^{(M, \{K_i\})}(x):= \frac{1}{M}\sum_{i=1}^M g\left(\bar{X}_{K_i}(x, Y_{i,1},\dots, Y_{i, K_i})\right),
    \]
    where $M \geq 1$, $K_i\geq 1$ for $i = 1,\dots, M$, and $Y_{i,k},\; k = 1,\dots, K_i$ are unit vectors. Then, for any $\delta \in (0,1)$, there exist $M $ and $K_i$, $1\leq i\leq M$, all polynomial in $\delta^{-1}$ \cite{grohs2022deep}, such that 
    \[
       \sqrt{\int_\Omega \left(\mathbb{E}_x\left[g(X_{\tau(\Omega)})\right] - g^{(M, \{K_i\})}(x)\right)^2 dx }\leq \delta.
    \]
    
    Similarly, we consider approximating the second term in \eqref{eq-sol-poisson} using the Monte-Carlo sampling technique. Let $Y_{i,k},\; k = 1,\dots, K_i$  and $y_{i,j,k},\; k = 1,\dots, K_i,\; j = 1,\dots, J$ be random unit vectors, for $i = 1,\dots, M$, with $M, K_i,J$ being positive integers, then for any $\delta > 0$, one can show that the following function $f^{(M,\{K_i\},J)}:\mathbb{R}^d\to \mathbb{R}$ defined by
    \[
        \begin{aligned}
            f^{(M,\{K_i\},J)}(x) 
            :=  \frac{1}{MJ}\sum_{i=1}^M\sum_{k=1}^{M_i}\sum_{j=1}^J\mathcal{D}\left(\bar{X}_{k}(x, Y_{i,1},\dots, Y_{i, k})\right)^2
              \times  f\left(Z_{k-1}\right),
        \end{aligned}
    \]
    where 
    \[
        \begin{aligned}
            Z_{k-1}:= \bar{X}_{k-1}(x, Y_{i,1},\dots, Y_{i, k-1}) 
             + \mathcal{D}(\bar{X}_{k-1}(x, Y_{i,1},\dots, Y_{i, k-1}))y_{i,j,k},
        \end{aligned}
    \]
    can approximate the second term \eqref{eq-sol-poisson} in sense that 
    \[
        \sqrt{\int_\Omega \left(\mathbb{E}_x\left[\sum_{k\geq 1}\mathcal{D}(\bar{X}_{k})^2E_0\right] - f^{(M,\{K_i\},J)}(x) \right)^2dx}\leq \delta,
    \]
    where 
    \[
        E_0:= \mathbb{E}_0\left[\int_0^{\tau(B(0,1))}f\left(\bar{X}_{k-1}+\mathcal{D}(\bar{X}_{k-1})\cdot t\right)dt\right],
    \]
    if one chooses suitable $M,\; J$ and $K_i,\; i = 1,\dots, M$ all polynomial in $\delta^{-1}$ \cite{grohs2022deep}. Therefore, the proof is completed.
\end{proof}

The objective of the above theorem is to establish a proof-of-concept framework to guide symbolic discovery of PDE solutions using modern machine learning approaches, including LLMs. This theorem offers a partial constructive resolution to the open research question posed in the previous section: it identifies the specific operators needed to represent a $\delta$-accurate approximation of the PDE solution $u$ in terms of the operators used to define the source term $f$, the boundary condition $g$, and the domain geometry via the distance function $\mathcal{D}$. The central insight is that an accurate approximation $\tilde{u}$ can be constructed using only a limited set of algebraic operations, along with the computational building blocks inherent in $f$, $g$, and $\mathcal{D}$. This result carries several important implications. (1) In contexts where one seeks symbolic models for physical phenomena, the theorem guarantees the existence of a low-complexity symbolic approximation of $u$ that relies solely on a structured and interpretable set of primitive operations. (2) It also informs the design of neural architectures for PDE approximation, suggesting that imposing operator priors, i.e., by reusing the operators present in $f$ and $g$, is sufficient to achieve arbitrarily high approximation accuracy, thereby promoting efficiency and interpretability. (3) The polynomial bound on computational complexity with respect to $\delta^{-1}$ underscores the practical feasibility and scalability of such approximations. Taken together, these insights lay a foundation for a broader theory of \textit{operator-preserving approximations} for general PDEs, with significant potential impact in scientific machine learning, computational physics, and algorithmic model design. We plan to pursue a unified analysis for more general classes of PDEs in future work, and our current numerical results on linear conservation laws already provide clear evidence supporting this direction.

\section{Methodology} \label{sec:LLM} 

In this section, we propose a novel method that fine-tunes large language models (LLMs) to predict the operator sets present in the symbolic expressions of partial differential equation (PDE) solutions. Our approach consists of three main stages: data generation, model fine-tuning, and performance evaluation. Below, we detail the pipeline for generating synthetic data, the procedure for fine-tuning the LLMs, and the evaluation of their predictive capabilities.

\subsection{Binary Computational Trees for Expressions} \label{subsection-exp-tree}
Expressions in PDEs are represented as \textbf{binary computational trees}, as depicted in Figure \ref{fig:comp_expr_tree}. These trees provide a structured and hierarchical way to encode mathematical expressions, enabling efficient exploration and manipulation of the solution space. 

The construction of the tree proceeds in two main stages.
\textbf{Depth Specification}: The maximum depth of the tree is specified in advance, determining the number of hierarchical levels and thereby controlling the expressive capacity of the resulting representations. While deeper trees enable the generation of more complex expressions, they also entail higher computational costs.
\textbf{Recursive Construction}: Given the specified depth, the tree is constructed recursively. Starting from the root node, each node is expanded by generating child nodes according to predefined operator rules. 

Each node in the computational tree encodes an operator, which may be either a \textit{unary} or \textit{binary} operator.
\textbf{Unary Operators}: Unary operators act on a single operand and include a variety of nonlinear transformations such as trigonometric functions ($\sin(\cdot)$, $\cos(\cdot)$, $\tan(\cdot)$), exponential functions ($\exp(\cdot)$), logarithmic functions ($\lg(\cdot)$, $\ln(\cdot)$), and other elementary operations (e.g., $\sqrt{\cdot}$, $|\cdot|$, and power functions $(\cdot)^k$ for $k \in \mathbb{Z}$). Each unary operator node is augmented with two scalar parameters, $\alpha$ and $\beta$, which enable the composition of the operator with an affine transformation, yielding expressions of the form $\alpha \cdot u(\cdot) + \beta$, where $u(\cdot)$ denotes the unary operation. This parameterization allows the tree to flexibly represent both nonlinear and linearly adjusted transformations, thereby enhancing its expressiveness.
\textbf{Binary Operators}: Binary operators operate on two operands and include the standard arithmetic operations: addition ($+$), subtraction ($-$), multiplication ($\times$), and division ($/$). Each binary operator node combines the outputs of its left and right child subtrees to produce more complex expressions. For example, a multiplication node may combine two distinct sub-expressions, thereby enabling the construction of richer symbolic forms through hierarchical composition.

\subsection{Data Generation for Equation Types} \label{subsec:data_gen}

We next describe the data generation process for constructing a structured dataset of symbolic expressions derived from various PDE types. The data generation process is outlined in Figure \ref{fig:data-pipeline-new}.

\begin{figure}[htb!]
    \centering
    \includegraphics[width=0.98\textwidth]{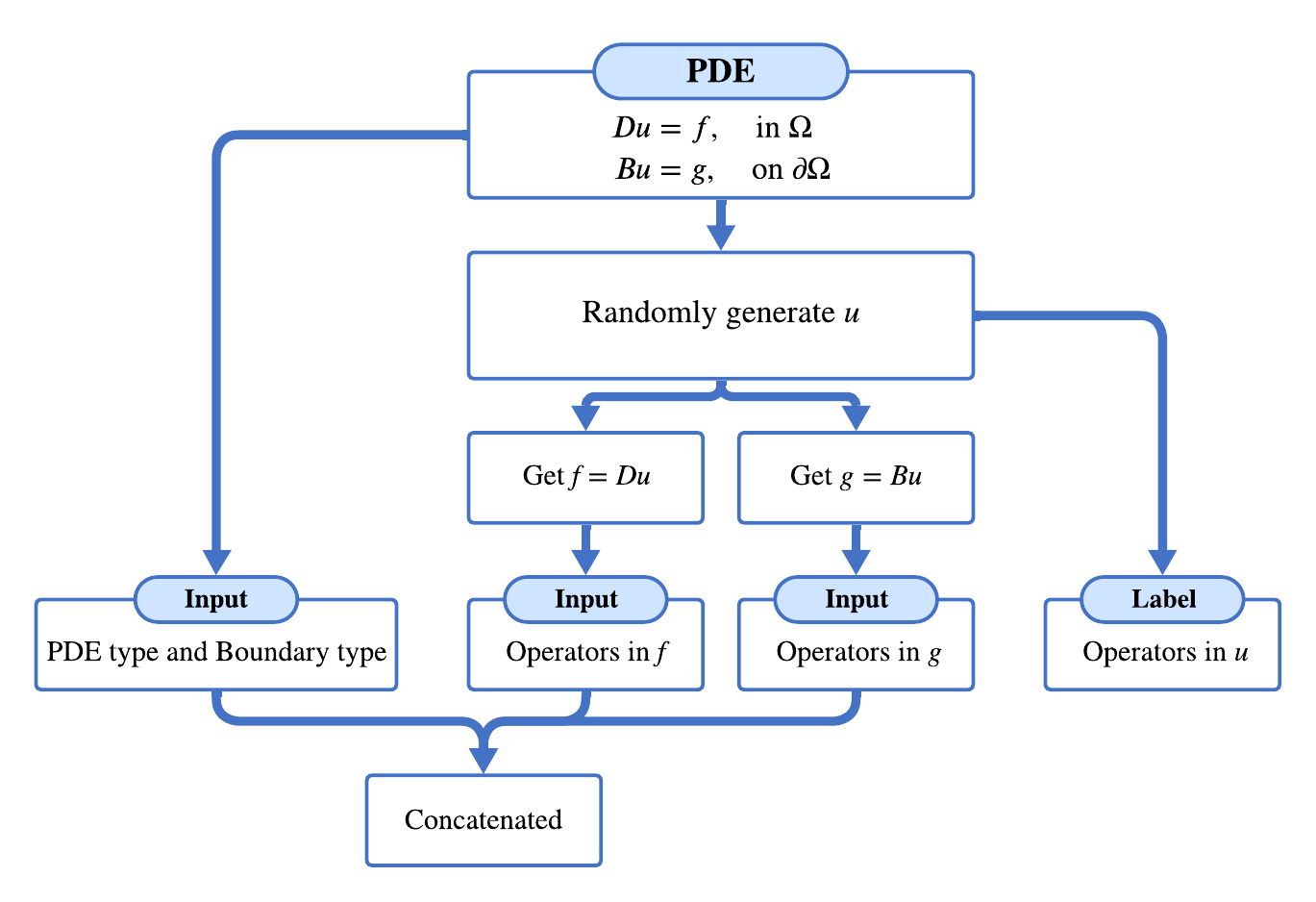}
    \caption{Data generation pipeline.}
    \label{fig:data-pipeline-new}
\end{figure}

Specifically, the construction of a single data point consists of the following steps: \textbf{PDE Type Specification.}
We begin by specifying the type of PDE and the type of boundary condition, based on the associated differential and boundary operators $D$ and $B$. This information forms the first component of the data point, identifying the PDE and boundary condition types. \textbf{Generation of Random Function $u$.}
Next, we construct a random solution function $u$ using a binary computational tree of predefined depth, as described in Section~\ref{subsection-exp-tree}. The tree is populated with randomly selected unary and binary operators from a sufficiently expressive operator set. This yields a symbolic expression for $u$, from which we compute the corresponding right-hand-side function $f$ and boundary function $g$, together forming the second component of the data point. \textbf{Postfix Representation.}
To facilitate processing by large language models (LLMs), the expressions for $u$, $f$, and $g$ are converted into postfix notation. This format simplifies tokenization, avoids the use of parentheses, and explicitly encodes operator precedence through position. Importantly, it captures the structural relationships among operators, which helps LLMs learn symbolic patterns. The postfix expressions of $f$ and $g$ constitute the model input, while the postfix form of $u(x)$ serves as the output label.

\subsection{Fine-tuning LLM}\label{subsec:train_llm}
With a generated data set, we can then fine-tune a certain LLM to learn the symbolic relation in PDEs. The whole pipeline is summarized in Figure \ref{fig:fine_tune_pipeline}, and we will describe each component in detail for the rest of this subsection.

\begin{figure*}[htb!]
  \centering
  \includegraphics[width=0.98\textwidth]{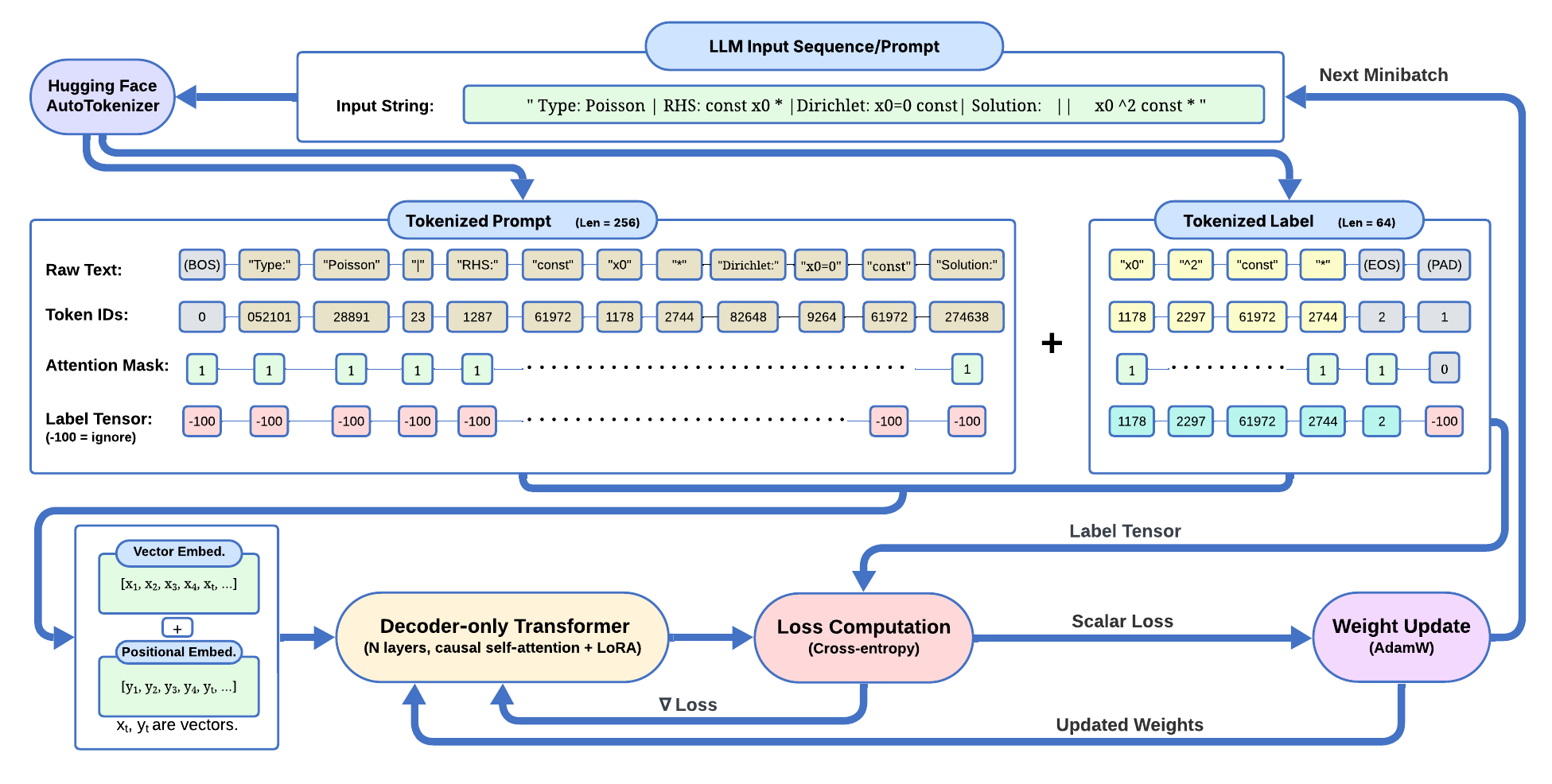}
  \caption{Overview of the fine-tuning pipeline}
  \label{fig:fine_tune_pipeline}
\end{figure*}

\textbf{Model Selection.} Encoder-decoder and decoder-only architectures are two fundamental types of LLMs, each with distinct design principles and strengths. \emph{Encoder-decoder models} (e.g., T5 \cite{raffel2020exploring} and BART \cite{lewis2019bart}) consist of an encoder that transforms the entire input sequence into a rich latent representation, and a decoder that generates the output sequence token by token, attending to the encoder's outputs via cross-attention. This architecture is well-suited for tasks such as machine translation, summarization, and question answering, where understanding the full input before generating the output is crucial. In contrast, \emph{decoder-only models} (e.g., Llama3 \cite{grattafiori2024llama}) treat the input and output as a single continuous sequence and generate tokens autoregressively using causal self-attention, meaning each token can only attend to previous tokens. This design simplifies the model architecture and training process, and has proven particularly effective for open-ended context generation and few-shot prompting. Based on our empirical results, decoder-only models are preferable for our task, as they more accurately predict the symbolic operators associated with PDE solutions.

\textbf{Tokenization.} We next perform tokenization to convert raw text into tokens that the model can process numerically, enabling efficient representation and learning. Specifically, each data point (in raw text) is serialized into a prompt followed by its target operator sequence, as demonstrated in the following.
\[
    \begin{aligned}
        &\; \underbrace{\texttt{Type:}<\mathrm{PDE}>\,\texttt{|}\,\texttt{RHS:}<\mathrm{ops}>\,\texttt{|}\, \texttt{BC\_Type:}<\mathrm{bc}>\,\texttt{|}\,\texttt{Solution:}
    }_{\text{PROMPT}} \\
    &\; \underbrace{\texttt{<op}_{1}\ \dots\ \texttt{op}_{n}\,\texttt{<EOS>}\texttt{>}
    }_{\text{TARGET}}
    \end{aligned}
\]
In this work, we intentionally adopted a simple and minimal prompt design to isolate the core capabilities of the fine-tuned language model and avoid introducing prompt-specific biases. However, incorporating explicit task descriptions or more structured prompts could potentially enhance performance, and exploring such prompt engineering strategies is a valuable research direction. 

Next, a pretrained tokenizer is applied to convert the input text into its corresponding token sequence. In this work, we use the \texttt{SentencePiece}~\footnote{Available at \url{https://github.com/google/sentencepiece}.} tokenizer and explicitly set the PAD token to match the EOS token, ensuring that padding positions are never treated as valid outputs. To accommodate memory constraints, we truncate the prompt to a maximum of 256 tokens and the target to a maximum of 64 tokens, resulting in a concatenated sequence length of 300. Additionally, we construct a binary attention mask, where real tokens are marked with 1 and padding positions with 0. This mask is supplied to the model to ensure that self-attention computations ignore padded positions.

\textbf{Fine-tuning.} Given a pretrained decoder-only language model and a tokenized dataset, we proceed to the fine-tuning phase, where the model’s parameters are updated via stochastic gradient-based optimization. Specifically, we minimize the cross-entropy loss between the model's predicted token distributions and the ground truth tokens in the training data. This objective encourages the model to assign higher probabilities to correct next-token predictions, thereby adapting its knowledge to the target task or domain. For optimization, we use the \texttt{AdamW} optimizer, a variant of Adam \cite{kingma2014adam} that decouples weight decay from the gradient updates, which has been shown to improve generalization performance in transformer-based models. To further enhance the efficiency of training, we adopt several computational techniques. First, we employ mixed-precision training, which allows computations to be performed in lower precision (e.g., float16) while maintaining a decent level of stability and accuracy. This significantly reduces memory usage and accelerates training. Second, we leverage Low-Rank Adaptation (LoRA) \cite{hu2022lora}, a parameter-efficient fine-tuning technique that injects trainable low-rank matrices into specific layers (typically attention or feedforward layers) of the pretrained model while keeping the original model weights frozen. This approach greatly reduces the number of trainable parameters and memory footprint during training, making it especially suitable for large-scale models and resource-constrained environments. Together, these strategies enable effective and efficient fine-tuning of large language models on our tasks.

\subsection{Inferencing} \label{subsection-inferencing}
Once fine-tuned, the language model is capable of predicting the postfix representation of previously unseen partial differential equation (PDE) solutions. Given a new input instance, comprising the PDE type, its right-hand side, and the corresponding boundary conditions, the input is first tokenized and then fed into the model. The decoder will auto-regressively generate a sequence of tokens that represents the solution in postfix  notation, encoding the sequence of mathematical operators and operands that define the PDE solution.

Since the model’s raw output is a token sequence, a post-processing pipeline then extracts a clean and interpretable set of predicted operators. The first step involves identifying the unique set of operators by removing duplicate tokens from the sequence. Following this, we perform error checking to detect and eliminate misspelled or malformed tokens that may result from tokenization artifacts or character-level prediction errors (issues that can arise because the model learns statistical patterns rather than enforcing strict syntactic correctness). Such errors are particularly likely when dealing with rare or domain-specific mathematical symbols. Any invalid or unrecognized tokens are subsequently discarded to ensure the integrity of the final operator set.

To measure the model’s predictive accuracy, we compute the squared $\ell_2$-norm of the difference between the predicted and ground-truth operator sets. Specifically, we encode each operator set as a binary vector over a fixed dictionary of $n$ possible operators, where each vector component indicates the presence or absence of a given operator. This encoding transforms the comparison into a well-defined vector space problem. An illustrative example of this binary encoding scheme with
\[
    h_1(x_1,x_2):=  (5x_1)^2 + \sin(3x_2) \cdot x_2,\quad  h_2(x_1):= 5\exp(2x_1) \cdot (\cos(6x_1))^3,
\]
is showed in Table~\ref{tab:vector_encoding_examples}. Then, the squared distance between two such operator sets, represented by two binary vectors $y \in \mathbb{R}^n$ and $z \in \mathbb{R}^n$, can be defined as $\|y-z\|^2 = \sum_{i=1}^n \left( y_i - z_i \right)^2.$  Clearly, it measures the number of mismatched operators between the two operator sets. The resulted mismatch can then be computed as $7$, indicating that the two functions are very different from each other.

\begin{table*}[htb!]
\centering
\begin{tabular}{ll *{9}{c}}
\multicolumn{2}{c}{} & \multicolumn{9}{c}{\textbf{Binary Vector}} \\
\cmidrule(lr){3-11}
\textbf{} & \textbf{Operator Set} &
\rotatebox{0}{\fontsize{7pt}{6pt}\selectfont $x_1$} & 
\rotatebox{0}{\fontsize{7pt}{6pt}\selectfont $x_2$} & 
\rotatebox{0}{\fontsize{7pt}{6pt}\selectfont \^{}2} &
\rotatebox{0}{\fontsize{7pt}{6pt}\selectfont \^{}3} &
\rotatebox{0}{\fontsize{7pt}{6pt}\selectfont +} &
\rotatebox{0}{\fontsize{7pt}{6pt}\selectfont *} &
\rotatebox{0}{\fontsize{7pt}{6pt}\selectfont SIN} &
\rotatebox{0}{\fontsize{7pt}{6pt}\selectfont COS} &
\rotatebox{0}{\fontsize{7pt}{6pt}\selectfont EXP} \\
\midrule
$h_1$ & \texttt{[ $x_1$, $x_2$, \^{}2, +, *, SIN ]} & 1 & 1 & 1 & 0 & 1 & 1 & 1 & 0 & 0 \\[6pt]
$h_2$ & \texttt{[ $x_1$, \^{}3, *, COS, EXP ]} & 1 & 0 & 0 & 1 & 0 & 1 & 0 & 1 & 1 \\[5pt]
\bottomrule
\end{tabular}
\caption{Sample binary‐vector encoding for two example expressions.}
\label{tab:vector_encoding_examples}
\end{table*}

\section{LLM-Informed Finite Expression Method} \label{sec:fex}

In this section, we integrate the predictive operator set extracted from PDE data, namely the forcing $f$, boundary data $g$, and geometric/metrical descriptors of $\partial\Omega$, into an efficient symbolic-learning framework for solving PDEs: the finite-expression (FEX) method \cite{liang2022finite}. The LLM-informed FEX pipeline for solving PDEs is illustrated in Fig.~\ref{fig:llmfex}

\begin{figure}[htb!]
    \centering
    \includegraphics[width=0.98\linewidth]{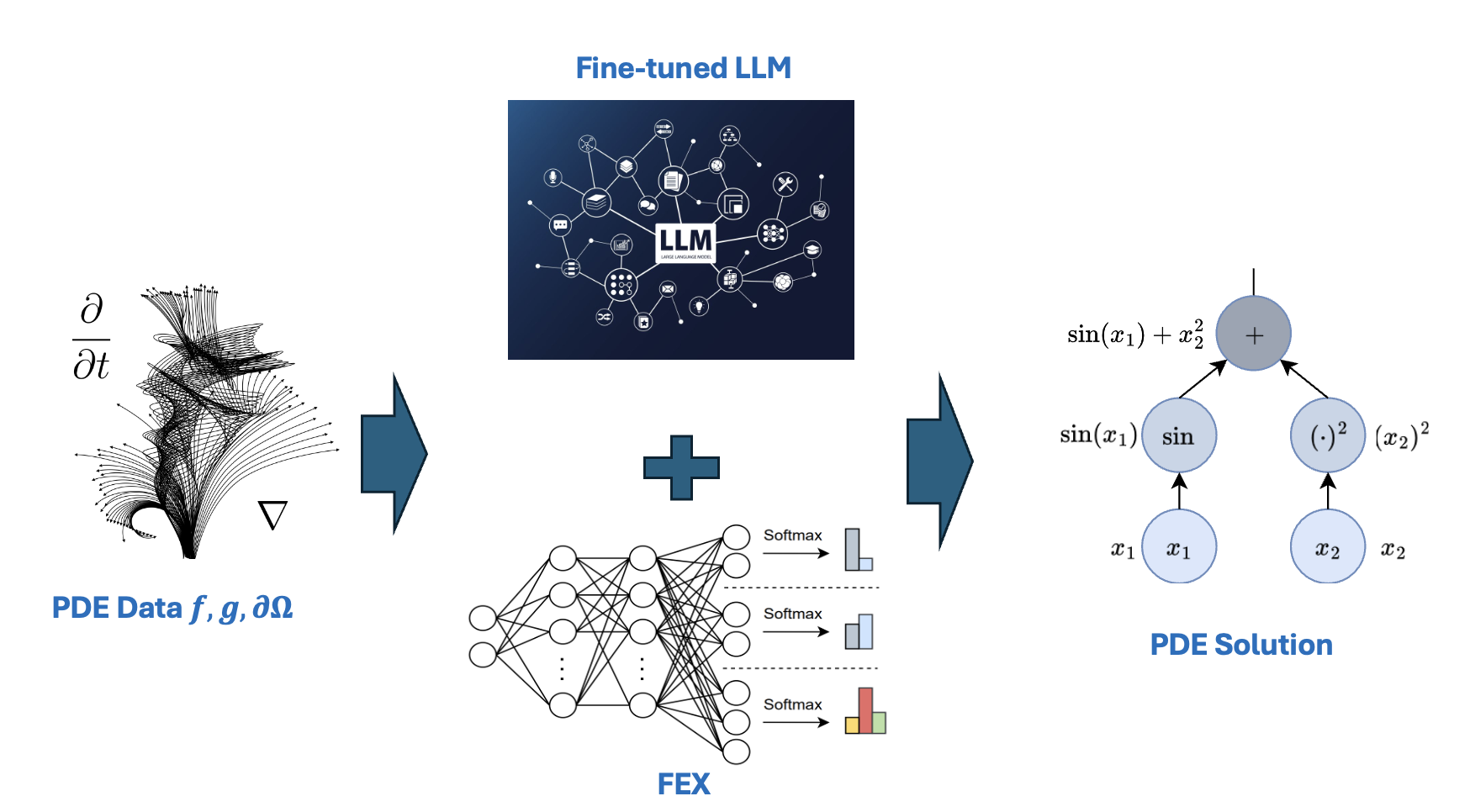}
    \caption{LLM-informed FEX for interpretable PDE solutions.}\label{fig:llmfex}
\end{figure}

FEX poses the construction of fully interpretable PDE solutions as a reinforcement learning search over an operator library; a policy samples symbolic programs whose execution yields an explicit surrogate $u_\theta$, and the return is a physics-informed score that penalizes interior residuals and boundary mismatches, as defined in \eqref{eq-least-squared}. Although FEX has demonstrated strong practical performance for high-dimensional PDEs, the convergence behavior of its policy-gradient optimization has remained unclear. We close this gap by establishing conditions under which policy-gradient iterates converge to stationary policies, thereby providing the novel convergence guarantees for symbolic PDE learning with policy gradients and placing the overall pipeline on firm theoretical footing among state-of-the-art alternatives.

After feeding the PDE data to the fine-tuned LLM, we obtain a predicted operator set that, together with a computational expression tree $\mathcal{T}$ (see Fig.~\ref{fig:comp_expr_tree}), induces a hypothesis class $\mathbb{S}$. We seek
\[
u^* \in \operatorname*{arg\,min}_{u\in \mathbb{S}} L(u).
\]
Writing $\mathbb{S}=\{u_\theta:\theta\in\Theta\subseteq\mathbb{R}^m\}$, each candidate $u\in\mathbb{S}$ is represented by a (possibly mixed discrete/continuous) parameter $\theta\in\Theta$. We model search over $\Theta$ with a stochastic policy $\theta\sim p_\phi$, where $p_\phi$ is a distribution on $\Theta$ parameterized by $\phi\in\Phi$ and $\Phi\subseteq\mathbb{R}^n$ is nonempty, closed, and convex. Defining $R(\theta):=(1+L(u_\theta))^{-1}\in[0,1]$, we can recast model selection as
\begin{equation}\label{eq-reward}
\max_{\phi\in\mathbb{R}^n}\; F(\phi)\;:=\; \mathbb{E}_{\theta\sim p_\phi}\!\big[R(\theta)\big]\;-\;\delta_\Phi(\phi),
\end{equation}
with $\delta_\Phi$ the indicator of $\Phi$ ($0$ on $\Phi$, $+\infty$ otherwise). We let $J(\phi):=\mathbb{E}_{\theta\sim p_\phi}[R(\theta)]$ for notational simplicity. Assuming $p_\phi$ has $\phi$-independent support, is differentiable in $\phi$, $R$ is integrable, and $\mathbb{E}_{p_\phi}\!\big[\|R(\theta)\nabla_\phi \log p_\phi(\theta)\|\big]<\infty$ (so that $\nabla$ and $\mathbb{E}$ interchange), the gradient of $J$ is easily obtained as follows:
\[
\nabla J(\phi)=\mathbb{E}_{\theta\sim p_\phi}\!\big[R(\theta)\,\nabla \log p_\phi(\theta)\big].
\]
Problem~\eqref{eq-reward} is generally nonconvex, so we can only target the first-order stationarity, without assuming restricted additional conditions. In particular, a point $\phi\in\mathbb{R}^n$ is stationary if
\[
0 \in -\nabla J(\phi) + N_\Phi(\phi),
\]
where $N_\Phi(\phi):=\{\,g\in\mathbb{R}^n:\langle g,\phi'-\phi\rangle\le 0,\ \forall\,\phi'\in\Phi\,\}$ is the normal cone of $\Phi$ at $\phi$ \cite{convexanalysis}. We call a (possibly random) iterate $\hat\phi$ an \emph{$\epsilon$-stationary point in expectation} if, for a sequence $\{\phi_t\}_{t=0}^T$ generated by a stochastic optimization method and a standard output rule, i.e., uniformly at random from $\{1,\dots,T\}$, it holds that
\[
\mathbb{E}_T\left[\mathrm{dist}\left(0, -\nabla J(\hat\phi) + N_\Phi(\hat\phi)\right)^2\right]\;\le\;\epsilon^2,
\]
where the expectation is over all algorithmic randomness (policy sampling, minibatching, and the random output selection).

With the above preparation, we can now apply the stochastic projected policy gradient method (SPPGM) for finding the optimal policy. The template of the SPPGM is presented in Algorithm \ref{alg-sppgm}.

\begin{algorithm}[htb!]
\caption{Stochastic Projected Policy Gradient Method (SPPGM)}
\label{alg-sppgm}
\begin{algorithmic}[1]
\Require Feasible set $\Phi\subseteq\mathbb{R}^n$, projection $\Pi_\Phi(\cdot)$, initial parameter $\phi_0\in\Phi$, step sizes $\{\eta_t\}_{t\ge0}$, batch size $B$
\For{$t=0,1,2,\ldots, T-1$}
    \State \textbf{Sampling:} Draw i.i.d.\ $\theta_1,\ldots,\theta_B \sim p_{\phi_t}$
    \State \textbf{Rewards:} Evaluate $r_i \gets R(\theta_i)$ for $i=1,\ldots,B$
    \State \textbf{Score gradients:} Compute $\;s_i \gets \nabla_\phi \log p_\phi(\theta_i)\big|_{\phi=\phi_t}$
    \State \textbf{Gradient estimate:} $g_t \;\gets\; \frac{1}{B}\sum_{i=1}^B r_i s_i$
    \State \textbf{Stochastic Gradient step:} $\;\tilde{\phi}_{t+1} \gets \phi_t + \eta_t\, g_t$
    \State \textbf{Projection:} $\;\phi_{t+1} \gets \Pi_\Phi\big(\tilde{\phi}_{t+1}\big)$ \Comment{$\Pi_\Phi(x):=\arg\min_{y\in\Phi}\;\tfrac{1}{2}\|y-x\|^2$}
    \State \textbf{Stopping:} if a criterion holds, \textbf{break}
\EndFor
\State \textbf{Output:} $\hat{\phi}_T$ is selected uniformly at random from the generated sequence $\{\phi_t\}_{t=1}^{T}$
\end{algorithmic}
\end{algorithm}

We end this section by presenting the following theorem, which shows that the SPPGM is able to return an $\epsilon$-stationary point in expectation with a constant learning rate by assume that $J(\phi)$ is $L$-smooth, i.e., for any $\phi',\phi$, there exists a constant $L>0$ such that 
\[
    \|\nabla J(\phi') - \nabla J(\phi)\|\leq L\|\phi' - \phi\|. 
\]
It is well-known (see, e.g., \cite{beck2017first}) that for any $L$-smooth function $J(\phi)$, it holds that 
\begin{equation}
    \label{eq-lsmooth}
    J(\phi')\geq J(\phi) + \nabla J(\phi)^T(\phi' - \phi) - \frac{L}{2}\|\phi' - \phi\|^2,\quad \forall\; \phi',\phi \in\mathbb{R}^n.
\end{equation}

\begin{theorem}
    Suppose that the function $J(\cdot)$ is $L$-smooth with $L>0$, $\eta_t = \eta\in (0, \tfrac{1}{2L})$ for all $t\geq 0$, and there exists a constant $\sigma$ such that
    \[
        \mathbb{E}_{\theta\sim p_\phi}\left[\|R(\theta)\nabla\log p_\phi(\theta) - \nabla J(\theta)\|^2\right]\leq \sigma^2
    \]
    then the Algorithm \ref{alg-sppgm} outputs a point $\hat\phi_T\in\Phi$ satisfying
    \begin{equation}
        \label{eq-convergence}
        \begin{aligned}
            &\; \mathbb{E}_T\left[\mathrm{dist}\left(0, -\nabla J(\hat\phi_T) + N_\Phi(\hat\phi_T)\right)^2\right] \\ 
            \leq &\; \left(2 + \frac{2}{\eta L(1-2\eta L)}\right)\frac{\sigma^2}{B} + \frac{\Delta }{T}\left(\frac{2}{\eta}+\frac{4}{\eta(1-2\eta L)}\right),
        \end{aligned}
    \end{equation}
    where $\Delta := F^*-F(\phi_0) > 0$ with $F^*$ being the global minimum of the problem \eqref{eq-reward}. Moreover, for any given $\epsilon > 0$, if one chooses 
    \[
        B := \left \lceil \frac{\sigma^2}{\epsilon^2}\left(4 + \frac{4}{\eta L(1-2\eta L}\right)\right \rceil, \quad T = \left \lceil \frac{\Delta}{\epsilon^2}\left(\frac{4}{\eta} + \frac{8}{\eta(1-2\eta L)}\right)\right \rceil,
    \]
    then $\mathbb{E}_T\left[\mathrm{dist}\left(0, -\nabla J(\hat\phi_T) + N_\Phi(\hat\phi_T)\right)^2\right] \leq \epsilon^2$, and the sample complexity is $\mathcal{O}\big(\epsilon^{-4}\big)$.
\end{theorem}

\begin{proof}
    From the $L$-smoothness of $J(\cdot)$ and \eqref{eq-lsmooth}, we see that 
    \begin{equation}
        \label{eq-pf-conv-1}
        J(\phi_{t+1})\geq J(\phi_t) + \nabla J(\phi_t)^T(\phi_{t+1} - \phi_t) - \frac{L}{2}\|\phi_{t+1} - \phi_t\|^2,\quad \forall t\geq 0.
    \end{equation}
    
    Since $\phi_{t+1} = \Pi_\Phi(\phi_t+\eta g_t)=\arg\min_{\phi\in\Phi}\,\tfrac{1}{2}\|\phi-(\phi_t+\eta g_t)\|^2$, by the first-order optimality condition of this projection problem, it is easy to verify that 
    \begin{equation}
        \label{eq-pf-conv-2}
        - g_t^T(\phi_{t+1}-\phi_t) + \frac{1}{2\eta}\|\phi_{t+1} - \phi_t\|^2 \leq 0,
    \end{equation}
    and that 
    \begin{equation}
        \label{eq-pf-conv-3}
        g^t - \frac{1}{\eta}(\phi_{t+1} - \phi_t) \in N_\Phi (\phi_{t+1}).
    \end{equation}
    From \eqref{eq-pf-conv-1} and \eqref{eq-pf-conv-2}, we have
    \[
        J(\phi_{t+1}) + g_t^T(\phi_{t+1} - \phi_t) - \frac{1}{2\eta}\|\phi_{t+1} - \phi_t\|^2 \geq J(\phi_t) + \nabla J(\phi_t)^T(\phi_{t+1} -\phi_t) - \frac{L}{2}\|\phi_{t+1} - \phi_t\|^2.
    \]
    Rearranging terms and using the fact that both $\phi_{t+1},\phi_t\in \Phi$ (hence, $\delta_{\Phi}(\phi_{t+1}) = \delta_{\Phi}(\phi_{t}) = 0$), we can rewrite the above inequality as 
    \begin{equation}
        \label{eq-pf-conv-4}
        \frac{1-\eta L}{2\eta}\|\phi_{t+1} - \phi_t\|^2 \leq F(\phi_{t+1}) - F(\phi_t) + (g_t - \nabla J(\phi_t))^T(\phi_{t+1} - \phi_t).
    \end{equation}

    On the one hand, by the Cauchy-Schwarz inequality, we see that 
    \[
        (g_t - \nabla J(\phi_t))^T(\phi_{t+1} - \phi_t) \leq \frac{1}{2L}\|g^t - \nabla J(\phi_t)\|^2 + \frac{L}{2}\|\phi_{t+1} - \phi_t\|^2,
    \]
    which together with \eqref{eq-pf-conv-4} implies that 
    \[
        \frac{1-2\eta L}{2\eta}\|\phi_{t+1} - \phi_t\|^2 \leq F(\phi_{t+1}) - F(\phi_t) + \frac{1}{2L}\|g^t - \nabla J(\phi_t)\|^2.
    \]
    Summing the above inequality over $t = 0,\dots, T-1$, we get 
    \begin{equation}\label{eq-pf-conv-6}
        \begin{aligned}
            &\; \frac{1-2\eta L}{2\eta}\sum_{t=0}^{T-1}\|\phi_{t+1} - \phi_t\|^2 \\
            \leq &\;  F(\phi_{T}) - F(\phi_0) + \frac{1}{2L}\sum_{t=0}^{T-1}\|g^t - \nabla J(\phi_t)\|^2 \\
            \leq &\; \Delta +  \frac{1}{2L}\sum_{t=0}^{T-1}\|g^t - \nabla J(\phi_t)\|^2.
        \end{aligned}
    \end{equation}

    On the other hand, adding the term $(\nabla J(\phi_{t+1}) - g_t)^T(\phi_{t+1}-\phi_t)$ and then multiplying $\tfrac{2}{\eta}$ on both sides of \eqref{eq-pf-conv-4}, we obtain that 
    \begin{equation}
        \label{eq-pf-conv-5}
        \begin{aligned}
            &\; 2(\nabla J(\phi_{t+1}) - g_t)^T\left(\frac{1}{\eta}(\phi_{t+1}-\phi_t)\right) + \frac{1-\eta L}{\eta^2}\|\phi_{t+1} - \phi_t\|^2 \\
            \leq &\; \frac{2}{\eta}(F(\phi_{t+1}) - F(\phi_t)) + \frac{2}{\eta}(\nabla J(\phi_{t+1}) - \nabla J(\phi_t))^T(\phi_{t+1} - \phi_t).
        \end{aligned}
    \end{equation}
    Notice that 
    \[
        \begin{aligned}
            &\; 2(\nabla J(\phi_{t+1}) - g_t)^T\left(\frac{1}{\eta}(\phi_{t+1}-\phi_t)\right) \\
            = &\; \left\lVert \nabla J(\phi_{t+1}) - g_t + \frac{1}{\eta}(\phi_{t+1}-\phi_t)\right\rVert^2 - \|\nabla J(\phi_{t+1}) - g_t\|^2 - \frac{1}{\eta^2}\|\phi_{t+1}-\phi_t\|^2,
        \end{aligned}
    \]
    we see that \eqref{eq-pf-conv-5} can be written as 
    \[
        \begin{aligned}
            &\; \left\lVert \nabla J(\phi_{t+1}) - g_t + \frac{1}{\eta}(\phi_{t+1}-\phi_t)\right\rVert^2 \\ 
            \leq &\; \|\nabla J(\phi_{t+1}) - g_t\|^2 + \left(\frac{1}{\eta^2}-\frac{1-\eta L}{\eta^2}\right)\|\phi_{t+1}-\phi_t\|^2 + \frac{2}{\eta}(F(\phi_{t+1}) - F(\phi_t)) \\
            &\; + \frac{2}{\eta}(\nabla J(\phi_{t+1}) - \nabla J(\phi_t))^T(\phi_{t+1} - \phi_t) \\
            \leq &\; 2\|\nabla J(\phi_{t}) - g_t\|^2 + 2 \|\nabla J(\phi_{t+1}) - \nabla J(\phi_t)\|^2 \quad \text{[$\|a-b\|^2\leq 2\|a-c\|^2 + 2\|b-c\|^2$]}\\
            &\; + \frac{L}{\eta}\|\phi_{t+1}-\phi_t\|^2 + \frac{2}{\eta}(F(\phi_{t+1}) - F(\phi_t)) \\
            &\; + \frac{2}{\eta}\|\nabla J(\phi_{t+1}) - \nabla J(\phi_t)\|\|\phi_{t+1} - \phi_t\|\quad \text{[Cauchy-Schwarz inequality]}\\
            \leq &\; 2\|\nabla J(\phi_{t}) - g_t\|^2 + \left(2L^2 + \frac{3L}{\eta}\right) \|\phi_{t+1}-\phi_t\|^2\quad \text{[$L$-smoothness]} \\
            &\;+ \frac{2}{\eta}(F(\phi_{t+1}) - F(\phi_t)).
        \end{aligned}
    \]
    Summing this inequality over $t = 0, \dots, T-1$, we get 
    \[
        \begin{aligned}
            &\; \sum_{t=0}^{T-1}\left\lVert \nabla J(\phi_{t+1}) - g_t + \frac{1}{\eta}(\phi_{t+1}-\phi_t)\right\rVert^2 \\
            \leq &\; 2 \sum_{t=0}^{T-1}\|\nabla J(\phi_{t}) - g_t\|^2 + \left(2L^2 + \frac{3L}{\eta}\right)\sum_{t=0}^{T-1} \|\phi_{t+1}-\phi_t\|^2 + \frac{2}{\eta}(F(\phi_T) - F(\phi_0)) \\ 
            \leq &\; 2 \sum_{t=0}^{T-1}\|\nabla J(\phi_{t}) - g_t\|^2 + \frac{2}{\eta^2}\sum_{t=0}^{T-1} \|\phi_{t+1}-\phi_t\|^2 + \frac{2\Delta }{\eta} \quad \text{[$L<\frac{1}{2\eta}$]} \\
            \leq &\; 2 \sum_{t=0}^{T-1}\|\nabla J(\phi_{t}) - g_t\|^2 + \frac{4}{\eta(1-2\eta L)}\left(\Delta +  \frac{1}{2L}\sum_{t=0}^{T-1}\|g^t - \nabla J(\phi_t)\|^2\right) + \frac{2\Delta }{\eta} \; \text{[\eqref{eq-pf-conv-6}]} \\
            = &\; 2 \sum_{t=0}^{T-1}\|\nabla J(\phi_{t}) - g_t\|^2 + \frac{2}{\eta L (1-2\eta L)} \sum_{t=0}^{T-1}\|g^t - \nabla J(\phi_t)\|^2 + \Delta \left(\frac{2}{\eta}+\frac{4}{\eta(1-2\eta L)}\right) \\
            =&\; 2\left(1 + \frac{2}{\eta L(1-2\eta L)}\right)\sum_{t=0}^{T-1}\|\nabla J(\phi_{t}) - g_t\|^2 + \Delta \left(\frac{2}{\eta}+\frac{4}{\eta(1-2\eta L)}\right),
        \end{aligned}
    \]
    which further implies that 
    \[
        \begin{aligned}
            &\; \mathbb{E}_T\left[\mathrm{dist}\left(0, -\nabla J(\hat\phi_T) + N_\Phi(\hat\phi_T)\right)^2\right] \\
            = &\; \frac{1}{T}\sum_{t=0}^{T-1} \mathbb{E}_T\left[\mathrm{dist}\left(0, -\nabla J(\phi_{t+1}) + N_\Phi(\phi_{t+1})\right)^2\right] \qquad \text{[output rule]}\\
            \leq &\; \frac{1}{T}\sum_{t=0}^{T-1} \mathbb{E}_T\left[\left\lVert\nabla J(\phi_{t+1}) - g_t + \frac{1}{\eta}(\phi_{t+1} - \phi_t)\right\rVert^2\right] \qquad \text{[projection optimality]} \\
            \leq &\; \frac{1}{T}\left(2 + \frac{2}{\eta L(1-2\eta L)}\right)\sum_{t=0}^{T-1}\mathbb{E}_T\left[\|\nabla J(\phi_{t}) - g_t\|^2\right] + \frac{\Delta}{T} \left(\frac{2}{\eta}+\frac{4}{\eta(1-2\eta L)}\right) \\
            \leq &\; \frac{1}{T}\left(2 + \frac{2}{\eta L(1-2\eta L)}\right)\sum_{t=0}^{T-1}\frac{\sigma^2}{B} + \frac{\Delta}{T} \left(\frac{2}{\eta}+\frac{4}{\eta(1-2\eta L)}\right)\quad \text{[mini-batch sampling]} \\
            = &\; \left(2 + \frac{2}{\eta L(1-2\eta L)}\right)\frac{\sigma^2}{B} + \frac{\Delta}{T} \left(\frac{2}{\eta}+\frac{4}{\eta(1-2\eta L)}\right).
        \end{aligned}
    \]
    This proves \eqref{eq-convergence} and the remaining statements follow as a direct consequence of \eqref{eq-convergence}. Hence, the proof is completed.
\end{proof}

\section{Experimental Results} \label{sec:results} 

In this section, we conduct extensive experiments to evaluate the effectiveness of our proposed approach. Our evaluation consists of two main components: (1) assessing the predictive performance of fine-tuned large language models (LLMs) in identifying operator sets within PDE solutions and (2) demonstrating the practical impact of these predictions in guiding operator selection within the finite expression method (FEX).  

First, we analyze the ability of fine-tuned LLMs to predict the correct operator sets in PDE solutions. The objective is to assess whether the models can effectively learn and generalize the symbolic relationships between PDE solutions and problem data. We compare the predicted operator sets against ground-truth labels to measure the accuracy of the learned representations. Additionally, we examine the impact of model architecture by fine-tuning several commonly used open-source LLMs. Second, we apply the fine-tuned LLMs to enhance the efficiency and accuracy of the FEX method by leveraging their predicted operator sets. This experiment attempts to achieve two goals. On one hand, we demonstrate that by obtaining more precise operator sets of smaller sizes, the computational costs spent in FEX can be significantly reduced. On the other hand, we evaluate whether FEX achieves improved accuracy in learning PDE solutions when provided with the correct prior information about the relevant operators. By comparing FEX solutions obtained with and without LLM-guided operator selection, we quantify improvements in both computational efficiency and numerical accuracy. Through these experiments, we aim to establish the viability of fine-tuned LLMs as effective tools for discovering symbolic relations in PDE solutions and enhancing PDE-solving techniques.

\subsection{Experimental Setup} \label{subsec:experiment-setup} 

For data generation, we consider the Poisson equation and the Linear Conservation Law, each paired with three commonly used boundary conditions: Cauchy, Dirichlet, and Neumann conditions. Using a tree depth of 3, we randomly generate a dataset of 198,000 equations, with 99,000 examples for each PDE type. This dataset is then used to fine-tune 4 state-of-the-art large language models: BART \cite{lewis2019bart}, T5 \cite{raffel2020exploring}, LLaMA-3B and LLaMA-8B \cite{grattafiori2024llama}.  

For fine-tuning, we minimize the cross-entropy loss using the \texttt{AdamW} optimizer with a learning rate of $2 \times 10^{-4}$ and a weight decay coefficient of $0.01$. Models are trained for 12 epochs with a batch size of 32. Training is performed on 8 NVIDIA A6000 GPUs, each equipped with 48~GB of memory. To improve training efficiency and scalability, we employ mixed-precision training (BF16) and gradient checkpointing. We also apply Low-Rank Adaptation (LoRA) with a rank of 16, scaling factor of $ 32$, and dropout rate of 0.1. Model evaluation and checkpointing are conducted at the end of each epoch.

Finally, we integrate the model into the FEX solver.\footnote{For a detailed description of the FEX method, see \cite{liang2022finite}; for the original implementation, visit \url{https://github.com/LeungSamWai/Finite-expression-method}.} During the execution of FEX on a given PDE, the fine-tuned model is called for inference. The resulting expression is post-processed to extract the sets of unique unary and binary operators, which are then used to define the operator pool available to FEX during its symbolic search for an analytical solution. We refer to this enhanced version as \emph{LLM-informed FEX}, in contrast to the original \emph{uninformed FEX}. The uninformed FEX employs a fixed operator set:  
\[
\begin{aligned}
    &\text{Binary set: } \mathbb{B} = \{+, -, *\}, \\
    &\text{Unary set: } \mathbb{U} = \{0, 1, \text{Id}, (\cdot)^2, (\cdot)^3, (\cdot)^4, \exp, \sin, \cos\}.
\end{aligned}
\]
In contrast, the LLM-informed FEX dynamically predicts the operator sets for each example, enabling a more adaptive and potentially efficient search process.  

\subsection{Effectiveness of Fine-tuned LLMs} \label{subsec:llm-effectiveness}
In this subsection, we evaluate the performance of the fine-tuned BART, T5 and Llama3 models in predicting operator sets. Specifically, we track the average number of mismatched operators on the test dataset for each epoch. The computational results are summarized in Figure~\ref{fig:mse}.  
\begin{figure}[htb!]
    \centering
    \includegraphics[width=0.75\linewidth]{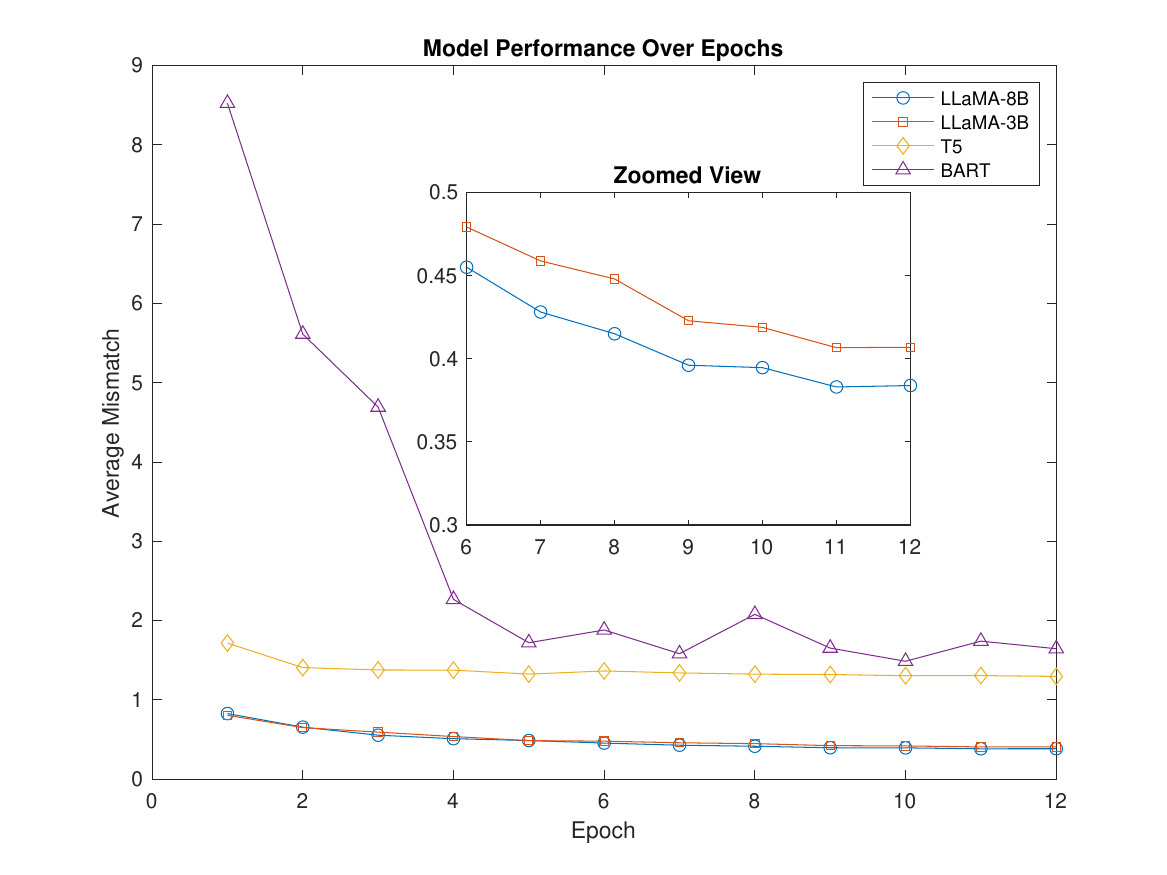}
    \caption{Comparison between T5, BART, Llama3-3B and Llama3-8B in terms of average mismatch on test dataset.}\label{fig:mse}
\end{figure}
Interested readers are referred to the Table \ref{tab:comparison_mismatch} in Appendix \ref{appendix-mismatch} for the detailed computational results.

Based on the computational results, LLaMA-8B consistently achieves the lowest number of average mismatch, starting at 0.83 and decreasing to 0.38 within 12 epochs, demonstrating superior performance and training stability. The smaller LLaMA-3B variant also performs competitively, with slightly higher values than LLaMA-8B at each epoch, further confirming the effectiveness of the LLaMA architecture. These results indicate that, with fine-tuning, LLaMA models can accurately predict the operators in PDE solutions on average. In contrast, T5 starts with a significantly higher initial loss of 1.72, converging more slowly to 1.30, suggesting less efficient learning. BART begins with the highest initial loss (8.53), which decreases rapidly but still remains substantially higher than that of the LLaMA models by the final epoch (1.65). Overall, the decoder-only LLaMA models, particularly LLaMA-8B, exhibit more effective and stable learning dynamics compared to the encoder-decoder models T5 and BART, with larger model sizes empirically yielding better performance.

\subsection{Effectiveness of LLM-informed FEX} \label{sec:LLM-informed-FEX}

In this section, we compare the performance of LLM-informed FEX with the fine-tuned LLaMA-8B model and the (original) uninformed FEX. To this end, we evaluated the performance of the FEX algorithm on a set of 100 randomly generated PDE examples, including 50 Poisson equations and 50 Linear Conservation Law equations. Since FEX is a randomized reinforcement learning approach, each instance was solved five times to obtain average computational results, including the number of iterations until convergence, computation time, and solution error. The detailed computational results can be found in Appendix \ref{appendix-fex}, and Table \ref{tab:randpde} and Figure~\ref{fig:accuracy} summarize the statistics of these results. 

\begin{table}[htb!]
\centering
\small
\setlength{\tabcolsep}{1.65pt} 
\begin{tabular}{llccc}
\toprule
\textbf{PDE} & \textbf{Metric} & \textbf{Uninformed} & \textbf{LLM-Informed} & \textbf{Speedup} \\
\midrule
\multirow{2}{*}{Conserv} 
 & Avg. Iter     & 28.52 &  5.23 & $5.45\times$ \\
 & Avg. Time (m) & 23.67 &  3.95 & $6.00\times$ \\
\midrule
\multirow{2}{*}{Poisson} 
 & Avg. Iter     & 28.73 &  6.47 & $4.44\times$ \\
 & Avg. Time (m) & 62.07 & 14.09 & $4.41\times$ \\
\bottomrule
\end{tabular}
\caption{Comparison of computational efficiency between uninformed and LLM-informed FEX.}
\label{tab:randpde}
\end{table}

\begin{figure}[htb!]
  \centering
  \begin{subfigure}[b]{0.48\linewidth}
    \centering
    \includegraphics[width=0.98\linewidth]{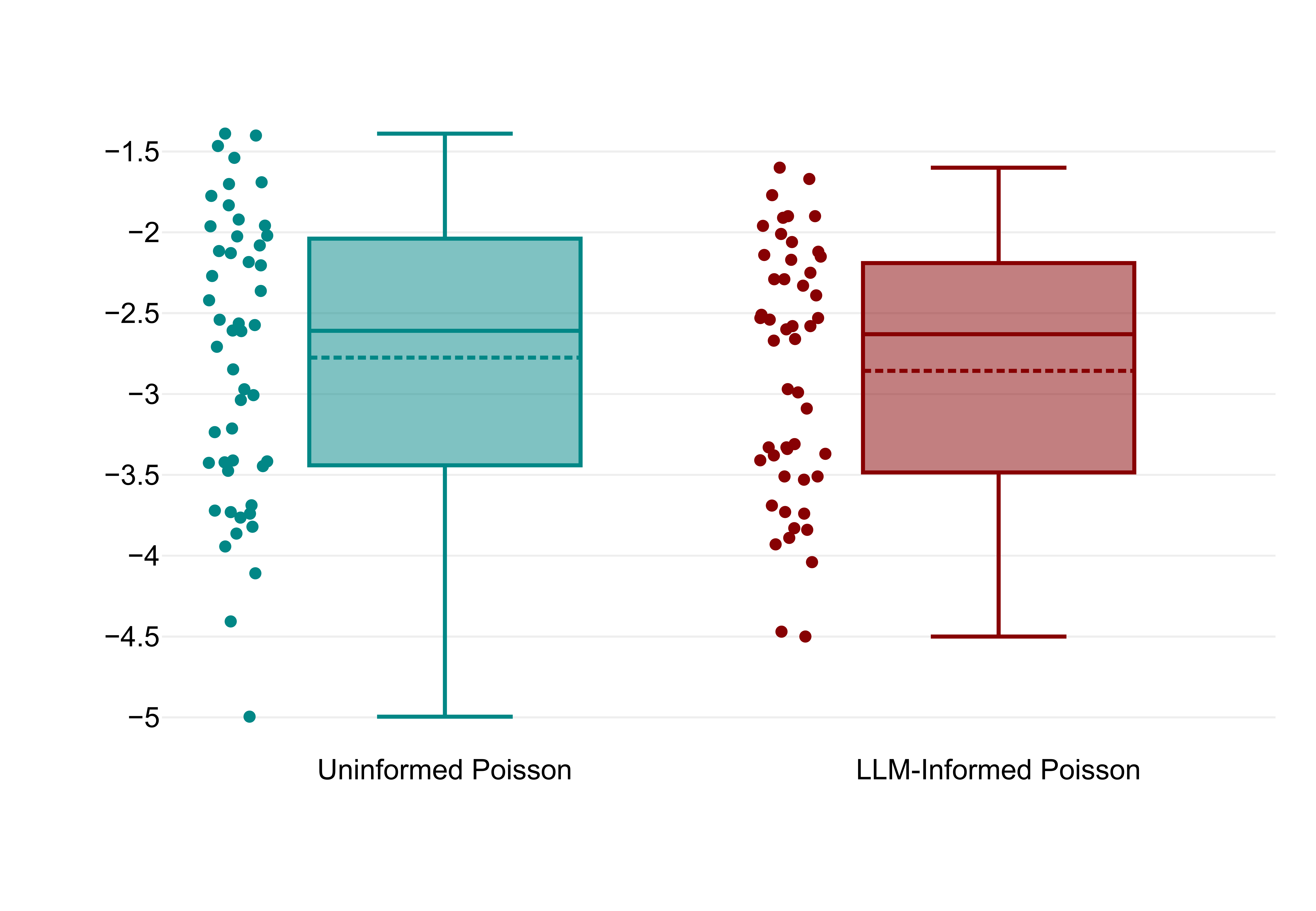}
    \label{fig:poisson-box}
  \end{subfigure}\hfill
  \begin{subfigure}[b]{0.48\linewidth}
    \centering
    \includegraphics[width=0.98\linewidth]{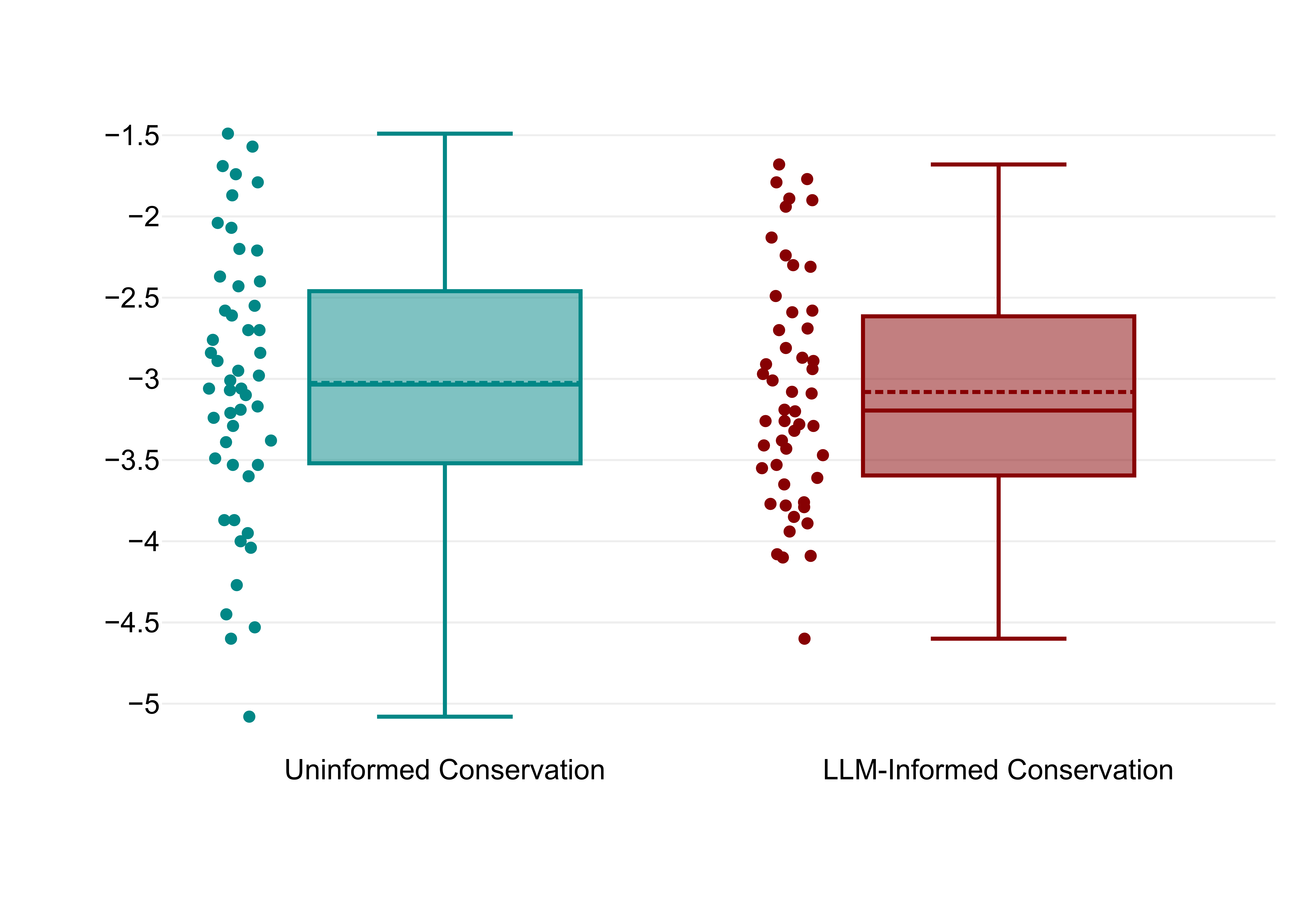}
    \label{fig:linconservation-box}
  \end{subfigure}
  \caption{Comparison on solution accuracy. The accuracy is in log-domain.}   
  \label{fig:accuracy}
\end{figure}

The computational results show that the LLM-informed FEX significantly improves efficiency over the uninformed FEX. A key advantage is the reduction in the number of operators used. The uninformed FEX always relies on a fixed set of 3 binary and 9 unary operators. In contrast, the LLM-informed FEX dynamically and effectively selects smaller and more relevant operator sets. This smalle and accurate search space leads to a substantial drop in the number of iterations needed for convergence. This efficiency gain further results in significant computational time savings. Despite this excellent speedup, the accuracy remains competitive: when the uninformed FEX converges, both methods produce similar approximation errors. And the LLM-informed FEX demonstrates better performance in terms of stablility as shown in Figure \ref{fig:accuracy}. These findings demonstrate that integrating the fine-tuned LLM into symbolic machine learning based PDE solvers enables more efficiency and reliability via informed operator selection. Future research directions include developing a unified theoretical framework for a broader class of PDEs, conducting extensive numerical experiments on complex real-world problems to evaluate practical effectiveness, and benchmarking existing symbolic regression methods or designing novel frameworks tailored to PDE solution discovery.

\section{Conclusions} \label{sec:conclusions}

In this work, we explored the potential of large language models (LLMs) to uncover symbolic relationships in partial differential equations (PDEs), a largely unexplored challenge in the intersection of AI and scientific computing. Based on our theoretical insights on Poisson equation, our results demonstrate that fine-tuned LLMs can effectively predict the operators involved in PDE solutions by leveraging symbolic information from the governing equations. By integrating these predictions into symbolic machine learning approaches, such as the finite expression method (FEX), we significantly enhanced both the efficiency and accuracy of analytical PDE approximations. Compared to the traditional uninformed FEX, the LLM-informed approach reduces the number of operators required, accelerates convergence, and maintains high solution accuracy, providing a fully interpretable and computationally efficient pipeline. These findings highlight the promising role of LLMs in advancing symbolic reasoning for scientific problems, paving the way for further exploration of AI-driven methodologies in mathematical modeling and equation solving.

\section*{Acknowledgments}
The authors were partially supported by the US National Science Foundation under awards DMS-2244988, IIS-25209787, GEO-5239902, and the DARPA D24AP00325-00.

\bibliographystyle{siamplain}
\bibliography{references}

\appendix

\section{Detailed Computational Results for Average Mismatch} \label{appendix-mismatch}
See Table~\ref{tab:comparison_mismatch}.

\begin{table*}[htb!]
    \centering
    \begin{tabular}{ccccc}
        \toprule
        \textbf{Epoch} & \textbf{LLaMA-8B} & \textbf{LLaMA-3B} & \textbf{T5} & \textbf{BART} \\
        \midrule
        1  & 0.82697 & 0.80510 & 1.71635 & 8.52542 \\
        2  & 0.65606 & 0.65040 & 1.40708 & 5.61399 \\
        3  & 0.55470 & 0.59328 & 1.37730 & 4.69332 \\
        4  & 0.51061 & 0.53747 & 1.37543 & 2.26791 \\
        5  & 0.48742 & 0.48601 & 1.32494 & 1.72084 \\
        6  & 0.45515 & 0.47934 & 1.36476 & 1.88165 \\
        7  & 0.42818 & 0.45894 & 1.34016 & 1.58241 \\
        8  & 0.41510 & 0.44808 & 1.32464 & 2.07886 \\
        9  & 0.39616 & 0.42293 & 1.31991 & 1.65313 \\
        10 & 0.39470 & 0.41894 & 1.30461 & 1.48738 \\
        11 & 0.38308 & 0.40672 & 1.30633 & 1.74097 \\
        12 & 0.38394 & 0.40697 & 1.29627 & 1.64541 \\
        \bottomrule
    \end{tabular}
    \caption{Comparison of Average Mismatch Across Models During Fine-tuning on 198K Dataset}
    \label{tab:comparison_mismatch}
\end{table*}

\section{Detailed computational results for FEX} \label{appendix-fex}
\begin{table*}[htb!]
\centering
\tiny
\makebox[0.85\textwidth][c]{%
\begin{tabular}{lrrrrrr}
\toprule
True solution & Iter (Orig) & Iter (LLM) & Time (Orig) & Time (LLM) & Err (Orig) & Err (LLM) \\
\midrule
$2\bigl(-2x_2^3+2\sin(x_3)\bigr)^3$            & 18.2  & 3.2  & 42.932   & 4.500   & $1.20\times10^{-2}$ & $5.09\times10^{-3}$ \\
$2-2\exp\bigl(4x_2\cos(x_2)\bigr)$             & 77.2  & 15.4 & 160.186  & 41.838  & $2.89\times10^{-2}$ & $2.89\times10^{-3}$ \\
$8x_2^4-2\bigl(2\cos(x_2)+2\bigr)^2$           & 100.0 & 20.0 & 231.106  & 52.544  & $9.44\times10^{-3}$ & $2.13\times10^{-3}$ \\
$-4x_2^4+4x_2^3+2$                            & 15.8  & 3.2  & 31.420   & 5.098   & $5.81\times10^{-4}$ & $1.03\times10^{-3}$ \\
$2-2\bigl(-2\exp(x_1)+2\cos(x_1)\bigr)^2$      & 38.0  & 7.8  & 82.448   & 15.878  & $2.88\times10^{-3}$ & $2.21\times10^{-3}$ \\
$2\sin\bigl(2x_3^3-2\exp(x_3)\bigr)$           & 35.6  & 8.0  & 83.972   & 19.326  & $1.51\times10^{-4}$ & $4.52\times10^{-4}$ \\
$-4x_2^4+4x_3^4+8\cos(x_3)$                   & 25.6  & 5.8  & 62.552   & 13.770  & $1.42\times10^{-3}$ & $2.92\times10^{-4}$ \\
$2\sin\bigl(2x_3^3+2\cos(x_3)\bigr)$           & 13.6  & 3.4  & 21.496   & 5.974   & $3.35\times10^{-4}$ & $1.46\times10^{-4}$ \\
$2\exp\bigl(2x_2^2-2\sin(x_2)\bigr)$           & 24.2  & 6.2  & 62.858   & 12.180  & $2.04\times10^{-2}$ & $1.71\times10^{-2}$ \\
$64\sin^2(x_2)\cos^2(x_1)$                    & 32.8  & 9.8  & 78.910   & 22.980  & $1.09\times10^{-2}$ & $2.66\times10^{-3}$ \\
$4\exp\bigl(2\sin(x_2)+2\bigr)$                & 5.2   & 1.6  & 8.154    & 1.280   & $5.37\times10^{-3}$ & $7.00\times10^{-3}$ \\
$512\sin^4(x_1)\cos^4(x_1)$                   & 49.4  & 15.4 & 113.850  & 40.554  & $3.42\times10^{-2}$ & $2.53\times10^{-2}$ \\
$2\exp(2x_2)-2\exp\bigl(2x_1^2-2x_2^2\bigr)$   & 70.8  & 22.4 & 148.142  & 49.740  & $4.08\times10^{-2}$ & $7.19\times10^{-3}$ \\
$2\exp\bigl(2x_3^4-2\exp(x_3)\bigr)$           & 17.0  & 5.4  & 32.546   & 10.676  & $2.73\times10^{-3}$ & $4.07\times10^{-3}$ \\
$4\bigl(2x_2^4+2x_2^3\bigr)^2$                 & 17.6  & 6.0  & 28.374   & 11.458  & $1.47\times10^{-2}$ & $1.10\times10^{-2}$ \\
$2\bigl(2x_2^4-2x_2^2\bigr)^3$                 & 51.6  & 17.6 & 106.732  & 42.008  & $2.45\times10^{-3}$ & $1.08\times10^{-3}$ \\
$-2\bigl(2x_1^2-2\bigr)^4$                    & 3.4   & 1.2  & 4.070    & 0.554   & $8.30\times10^{-3}$ & $1.25\times10^{-2}$ \\
$-2\sin\bigl(2\sin(x_2)-2\bigr)$               & 2.8   & 1.0  & 4.670    & 0.000   & $1.72\times10^{-4}$ & $1.49\times10^{-4}$ \\
$2\exp\bigl(2\cos(x_2)+2\bigr)$               & 5.6   & 2.0  & 8.260    & 2.170   & $2.47\times10^{-3}$ & $9.79\times10^{-3}$ \\
$4\cos\bigl(4x_2^4\bigr)\cos\bigl(2x_1^4-2\bigr)$ & 40.4 & 15.8 & 91.274   & 38.016  & $6.25\times10^{-3}$ & $5.63\times10^{-3}$ \\
$2\sin\bigl(4\exp(x_1)\bigr)$                  & 5.8   & 2.4  & 9.644    & 4.926   & $2.67\times10^{-3}$ & $1.26\times10^{-2}$ \\
$2\bigl(2x_3^4-2\bigr)^3+2\cos\bigl(2x_2-2\bigr)$ & 76.8 & 31.8 & 175.690  & 78.252  & $1.99\times10^{-2}$ & $5.17\times10^{-3}$ \\
$16\sin(x_3)$                                 & 2.4   & 1.0  & 2.362    & 0.000   & $1.82\times10^{-4}$ & $8.07\times10^{-4}$ \\
$2\exp\bigl(2x_1^4+2\sin(x_3)\bigr)+2$         & 24.6  & 10.6 & 57.754   & 22.852  & $7.66\times10^{-3}$ & $8.73\times10^{-3}$ \\
$2\sin\bigl(2x_1^4-2\sin(x_3)\bigr)-2$         & 55.2  & 24.0 & 126.722  & 62.458  & $1.90\times10^{-4}$ & $1.30\times10^{-4}$ \\
$2\sin\bigl(2x_3+2\bigr)-2$                    & 2.2   & 1.0  & 3.464    & 0.000   & $1.86\times10^{-4}$ & $4.64\times10^{-4}$ \\
$2\cos\bigl(2\sin(x_2)+2\cos(x_3)\bigr)$       & 8.0   & 3.8  & 15.570   & 6.608   & $1.37\times10^{-4}$ & $9.22\times10^{-5}$ \\
$2\cos\bigl(2x_1^3-2x_3^2\bigr)$               & 20.6  & 10.0 & 55.294   & 23.446  & $9.84\times10^{-4}$ & $4.13\times10^{-4}$ \\
$2\exp\bigl(2x_1^2+2\sin(x_3)\bigr)+2$         & 17.2  & 8.4  & 39.844   & 19.168  & $9.55\times10^{-3}$ & $7.52\times10^{-3}$ \\
$32x_3^6+2$                                   & 2.0   & 1.0  & 2.084    & 0.000   & $1.10\times10^{-2}$ & $2.93\times10^{-3}$ \\
$2\bigl(-2x_2^3+2\sin(x_3)\bigr)^3$            & 18.2  & 3.2  & 42.932   & 4.500   & $1.20\times10^{-2}$ & $5.09\times10^{-3}$ \\
$2-2\exp\bigl(4x_2\cos(x_2)\bigr)$             & 77.2  & 15.4 & 160.186  & 41.838  & $2.89\times10^{-2}$ & $2.89\times10^{-3}$ \\
$8x_2^4-2\bigl(2\cos(x_2)+2\bigr)^2$           & 100.0 & 20.0 & 231.106  & 52.544  & $9.44\times10^{-3}$ & $2.13\times10^{-3}$ \\
$-4x_2^4+4x_2^3+2$                             & 15.8  & 3.2  & 31.420   & 5.098   & $5.81\times10^{-4}$ & $1.03\times10^{-3}$ \\
$2-2\bigl(-2\exp(x_1)+2\cos(x_1)\bigr)^2$      & 38.0  & 7.8  & 82.448   & 15.878  & $2.88\times10^{-3}$ & $2.21\times10^{-3}$ \\
$2\sin\bigl(2x_3^3-2\exp(x_3)\bigr)$           & 35.6  & 8.0  & 83.972   & 19.326  & $1.51\times10^{-4}$ & $4.52\times10^{-4}$ \\
$-4x_2^4+4x_3^4+8\cos(x_3)$                    & 25.6  & 5.8  & 62.552   & 13.770  & $1.42\times10^{-3}$ & $2.92\times10^{-4}$ \\
$2\sin\bigl(2x_3^3+2\cos(x_3)\bigr)$           & 13.6  & 3.4  & 21.496   & 5.974   & $3.35\times10^{-4}$ & $1.46\times10^{-4}$ \\
$2\exp\bigl(2x_2^2-2\sin(x_2)\bigr)$           & 24.2  & 6.2  & 62.858   & 12.180  & $2.04\times10^{-2}$ & $1.71\times10^{-2}$ \\
$64\sin^2(x_2)\cos^2(x_1)$                     & 32.8  & 9.8  & 78.910   & 22.980  & $1.09\times10^{-2}$ & $2.66\times10^{-3}$ \\
$4\exp\bigl(2\sin(x_2)+2\bigr)$                & 5.2   & 1.6  & 8.154    & 1.280   & $5.37\times10^{-3}$ & $7.00\times10^{-3}$ \\
$512\sin^4(x_1)\cos^4(x_1)$                    & 49.4  & 15.4 & 113.850  & 40.554  & $3.42\times10^{-2}$ & $2.53\times10^{-2}$ \\
$2\exp(2x_2)-2\exp\bigl(2x_1^2-2x_2^2\bigr)$   & 70.8  & 22.4 & 148.142  & 49.740  & $4.08\times10^{-2}$ & $7.19\times10^{-3}$ \\
$2\exp\bigl(2x_3^4-2\exp(x_3)\bigr)$           & 17.0  & 5.4  & 32.546   & 10.676  & $2.73\times10^{-3}$ & $4.07\times10^{-3}$ \\
$4\bigl(2x_2^4+2x_2^3\bigr)^2$                 & 17.6  & 6.0  & 28.374   & 11.458  & $1.47\times10^{-2}$ & $1.10\times10^{-2}$ \\
$2\bigl(2x_2^4-2x_2^2\bigr)^3$                 & 51.6  & 17.6 & 106.732  & 42.008  & $2.45\times10^{-3}$ & $1.08\times10^{-3}$ \\
$-2\bigl(2x_1^2-2\bigr)^4$                     & 3.4   & 1.2  & 4.070    & 0.554   & $8.30\times10^{-3}$ & $1.25\times10^{-2}$ \\
$-2\sin\bigl(2\sin(x_2)-2\bigr)$               & 2.8   & 1.0  & 4.670    & 0.000   & $1.72\times10^{-4}$ & $1.49\times10^{-4}$ \\
$2\exp\bigl(2\cos(x_2)+2\bigr)$                & 5.6   & 2.0  & 8.260    & 2.170   & $2.47\times10^{-3}$ & $9.79\times10^{-3}$ \\
$4\cos\bigl(4x_2^4\bigr)\cos\bigl(2x_1^4-2\bigr)$ & 40.4 & 15.8 & 91.274   & 38.016  & $6.25\times10^{-3}$ & $5.63\times10^{-3}$ \\
$2\sin\bigl(4\exp(x_1)\bigr)$                  & 5.8   & 2.4  & 9.644    & 4.926   & $2.67\times10^{-3}$ & $1.26\times10^{-2}$ \\
$2\bigl(2x_3^4-2\bigr)^3+2\cos\bigl(2x_2-2\bigr)$ & 76.8 & 31.8 & 175.690  & 78.252  & $1.99\times10^{-2}$ & $5.17\times10^{-3}$ \\
$16\sin(x_3)$                                  & 2.4   & 1.0  & 2.362    & 0.000   & $1.82\times10^{-4}$ & $8.07\times10^{-4}$ \\
$2\exp\bigl(2x_1^4+2\sin(x_3)\bigr)+2$         & 24.6  & 10.6 & 57.754   & 22.852  & $7.66\times10^{-3}$ & $8.73\times10^{-3}$ \\
\bottomrule
\end{tabular}}
\caption{Results of LLM-informed and original FEX runs on 50 randomly generated Linear Conservation Law functions.}\label{tab-con}
\end{table*}

\begin{table*}[htb!]
\centering
\tiny
\setlength{\tabcolsep}{4pt}
\makebox[0.85\textwidth][c]{%
\begin{tabular}{l r r r r r r}
\toprule
True solution & Iter (Orig) & Iter (LLM) & Time (Orig) & Time (LLM) & Err (Orig) & Err (LLM) \\
\midrule
$8\,x_{1}\,\mathrm{exp}(x_{2})$ & 42.4 & 1.2 & 74.44 & 0.5900 & $3.75\times10^{-4}$ & $4.91\times10^{-4}$ \\
$128\,x_{1}^{12} + 2$ & 40.2 & 1.2 & 70.704 & 0.516 & $3.97\times10^{-2}$ & $2.14\times10^{-2}$ \\
$-32\,x_{2}^{10}$ & 35.2 & 1.2 & 71.586 & 0.536 & $7.44\times10^{-3}$ & $6.71\times10^{-3}$ \\
$-8\,x_{1}^{4}\,x_{2}$ & 25.0 & 1.2 & 53.342 & 0.510 & $9.18\times10^{-4}$ & $3.86\times10^{-4}$ \\
$2\,\cos\bigl(4\,x_{2}^{6}\bigr)$ & 44.2 & 3.0 & 85.792 & 4.326 & $1.68\times10^{-2}$ & $1.23\times10^{-2}$ \\
$-8\,x_{2}\,\cos(x_{2}) + 2$ & 32.4 & 2.4 & 78.076 & 4.548 & $3.78\times10^{-4}$ & $3.06\times10^{-4}$ \\
$8\,x_{1}^{2}\,\cos(x_{2}) + 2$ & 42.8 & 3.2 & 85.55 & 6.154 & $2.05\times10^{-4}$ & $1.85\times10^{-4}$ \\
$-4\,\mathrm{exp}(x_{3}) + 4\,\sin(x_{1}) + 2$ & 20.2 & 1.6 & 43.712 & 2.150 & $1.14\times10^{-4}$ & $2.04\times10^{-4}$ \\
$2\bigl(2\,x_{2}^{2} - 2\,x_{3}^{4}\bigr)^{2} + 2$ & 31.8 & 3.0 & 69.054 & 3.400 & $6.55\times10^{-3}$ & $2.66\times10^{-3}$ \\
$-32\,\sin(x_{2})^{4}$ & 15.6 & 1.6 & 36.210 & 1.704 & $4.34\times10^{-3}$ & $4.66\times10^{-3}$ \\
$-4\,x_{2}^{3} - 4\,x_{3}^{4} + 16$ & 20.6 & 2.4 & 45.110 & 2.716 & $3.88\times10^{-4}$ & $4.68\times10^{-4}$ \\
$2\,\mathrm{exp}\bigl(-2\,\mathrm{exp}(x_{3}) + 2\,\cos(x_{3})\bigr) + 2$ & 22.0 & 2.8 & 52.848 & 4.994 & $3.83\times10^{-4}$ & $4.28\times10^{-4}$ \\
$8\,x_{1}^{2}\,x_{3} + 2$ & 11.4 & 1.6 & 26.022 & 1.742 & $3.58\times10^{-4}$ & $1.82\times10^{-4}$ \\
$2\,\sin\bigl(2\,\sin(x_{1}) - 2\,\cos(x_{3})\bigr)$ & 19.2 & 3.0 & 47.168 & 4.768 & $7.78\times10^{-5}$ & $3.13\times10^{-5}$ \\
$4\bigl(2\,\cos(x_{2}) - 2\,\cos(x_{3})\bigr)^{4}$ & 10.2 & 1.6 & 21.852 & 1.476 & $3.80\times10^{-3}$ & $3.12\times10^{-3}$ \\
$2\bigl(2\,x_{1} + 2\,\mathrm{exp}(x_{1})\bigr)^{2} - 2$ & 26.4 & 4.2 & 44.664 & 6.108 & $1.07\times10^{-3}$ & $2.94\times10^{-3}$ \\
$2\,\sin\bigl(2\,x_{2}^{3} + 2\,x_{3}^{4}\bigr) + 2$ & 29.0 & 4.8 & 73.818 & 12.376 & $6.12\times10^{-4}$ & $3.11\times10^{-4}$ \\
$2\bigl(-2\,x_{1} + 2\,\mathrm{exp}(x_{3})\bigr)^{2} - 2$ & 20.4 & 3.4 & 35.796 & 4.762 & $1.96\times10^{-3}$ & $2.52\times10^{-3}$ \\
$2\,\cos\bigl(2\,\mathrm{exp}(x_{1}) + 2\,\sin(x_{2})\bigr) + 2$ & 69.6 & 11.6 & 157.764 & 27.654 & $3.92\times10^{-5}$ & $1.17\times10^{-4}$ \\
$2\,\sin\bigl(2\,\sin(x_{2}) - 2\,\cos(x_{2})\bigr)$ & 18.2 & 3.2 & 37.792 & 5.254 & $1.01\times10^{-5}$ & $3.35\times10^{-5}$ \\
$2\bigl(-2\,x_{2}^{3} + 2\,\sin(x_{3})\bigr)^{3}$ & 18.2 & 3.2 & 42.932 & 4.500 & $1.20\times10^{-2}$ & $5.09\times10^{-3}$ \\
$2 - 2\,\mathrm{exp}\bigl(4\,x_{2}\,\cos(x_{2})\bigr)$ & 77.2 & 15.4 & 160.186 & 41.838 & $2.89\times10^{-2}$ & $2.89\times10^{-3}$ \\
$8\,x_{2}^{4} - 2\bigl(2\,\cos(x_{2}) + 2\bigr)^{2}$ & 100.0 & 20.0 & 231.106 & 52.544 & $9.44\times10^{-3}$ & $2.13\times10^{-3}$ \\
$-4\,x_{2}^{4} + 4\,x_{2}^{3} + 2$ & 15.8 & 3.2 & 31.420 & 5.098 & $5.81\times10^{-4}$ & $1.03\times10^{-3}$ \\
$2 - 2\bigl(-2\,\mathrm{exp}(x_{1}) + 2\,\cos(x_{1})\bigr)^{2}$ & 38.0 & 7.8 & 82.448 & 15.878 & $2.88\times10^{-3}$ & $2.21\times10^{-3}$ \\
$2\,\sin\bigl(2\,x_{3}^{3} - 2\,\mathrm{exp}(x_{3})\bigr)$ & 35.6 & 8.0 & 83.972 & 19.326 & $1.51\times10^{-4}$ & $4.52\times10^{-4}$ \\
$-4\,x_{2}^{4} + 4\,x_{3}^{4} + 8\,\cos(x_{3})$ & 25.6 & 5.8 & 62.552 & 13.770 & $1.42\times10^{-3}$ & $2.92\times10^{-4}$ \\
$2\,\sin\bigl(2\,x_{3}^{3} + 2\,\cos(x_{3})\bigr)$ & 13.6 & 3.4 & 21.496 & 5.974 & $3.35\times10^{-4}$ & $1.46\times10^{-4}$ \\
$2\,\mathrm{exp}\bigl(2\,x_{2}^{2} - 2\,\sin(x_{2})\bigr)$ & 24.2 & 6.2 & 62.858 & 12.180 & $2.04\times10^{-2}$ & $1.71\times10^{-2}$ \\
$64\,\sin(x_{2})^{2}\,\cos(x_{1})^{2}$ & 32.8 & 9.8 & 78.910 & 22.980 & $1.09\times10^{-2}$ & $2.66\times10^{-3}$ \\
$4\,\mathrm{exp}\bigl(2\,\sin(x_{2}) + 2\bigr)$ & 5.2 & 1.6 & 8.154 & 1.280 & $5.37\times10^{-3}$ & $7.00\times10^{-3}$ \\
$512\,\sin(x_{1})^{4}\,\cos(x_{1})^{4}$ & 49.4 & 15.4 & 113.850 & 40.554 & $3.42\times10^{-2}$ & $2.53\times10^{-2}$ \\
$2\,\mathrm{exp}(2\,x_{2}) - 2\,\mathrm{exp}(2\,x_{1}^{2} - 2\,x_{2}^{2})$ & 70.8 & 22.4 & 148.142 & 49.740 & $4.08\times10^{-2}$ & $7.19\times10^{-3}$ \\
$2\,\mathrm{exp}\bigl(2\,x_{3}^{4} - 2\,\mathrm{exp}(x_{3})\bigr)$ & 17.0 & 5.4 & 32.546 & 10.676 & $2.73\times10^{-3}$ & $4.07\times10^{-3}$ \\
$4\bigl(2\,x_{2}^{4} + 2\,x_{2}^{3}\bigr)^{2}$ & 17.6 & 6.0 & 28.374 & 11.458 & $1.47\times10^{-2}$ & $1.10\times10^{-2}$ \\
$2\bigl(2\,x_{2}^{4} - 2\,x_{2}^{2}\bigr)^{3}$ & 51.6 & 17.6 & 106.732 & 42.008 & $2.45\times10^{-3}$ & $1.08\times10^{-3}$ \\
$-2\bigl(2\,x_{1}^{2} - 2\bigr)^{4}$ & 3.4 & 1.2 & 4.070 & 0.554 & $8.30\times10^{-3}$ & $1.25\times10^{-2}$ \\
$-2\,\sin\bigl(2\,\sin(x_{2}) - 2\bigr)$ & 2.8 & 1.0 & 4.670 & 0.000 & $1.72\times10^{-4}$ & $1.49\times10^{-4}$ \\
$2\,\mathrm{exp}\bigl(2\,\cos(x_{2}) + 2\bigr)$ & 5.6 & 2.0 & 8.260 & 2.170 & $2.47\times10^{-3}$ & $9.79\times10^{-3}$ \\
$4\,\cos\bigl(4\,x_{2}^{4}\bigr)\,\cos\bigl(2\,x_{1}^{4} - 2\bigr)$ & 40.4 & 15.8 & 91.274 & 38.016 & $6.25\times10^{-3}$ & $5.63\times10^{-3}$ \\
$2\,\sin\bigl(4\,\mathrm{exp}(x_{1})\bigr)$ & 5.8 & 2.4 & 9.644 & 4.926 & $2.67\times10^{-3}$ & $1.26\times10^{-2}$ \\
$2\bigl(2\,x_{3}^{4} - 2\bigr)^{3} + 2\,\cos\bigl(2\,x_{2} - 2\bigr)$ & 76.8 & 31.8 & 175.690 & 78.252 & $1.99\times10^{-2}$ & $5.17\times10^{-3}$ \\
$16\,\sin(x_{3})$ & 2.4 & 1.0 & 2.362 & 0.000 & $1.82\times10^{-4}$ & $8.07\times10^{-4}$ \\
$2\,\mathrm{exp}\bigl(2\,x_{1}^{4} + 2\,\sin(x_{3})\bigr) + 2$ & 24.6 & 10.6 & 57.754 & 22.852 & $7.66\times10^{-3}$ & $8.73\times10^{-3}$ \\
$2\,\sin\bigl(2\,x_{1}^{4} - 2\,\sin(x_{3})\bigr) - 2$ & 55.2 & 24.0 & 126.722 & 62.458 & $1.90\times10^{-4}$ & $1.30\times10^{-4}$ \\
$2\,\sin\bigl(2\,x_{3} + 2\bigr) - 2$ & 2.2 & 1.0 & 3.464 & 0.000 & $1.86\times10^{-4}$ & $4.64\times10^{-4}$ \\
$2\,\cos\bigl(2\,\sin(x_{2}) + 2\,\cos(x_{3})\bigr)$ & 8.0 & 3.8 & 15.570 & 6.608 & $1.37\times10^{-4}$ & $9.22\times10^{-5}$ \\
$2\,\cos\bigl(2\,x_{1}^{3} - 2\,x_{3}^{2}\bigr)$ & 20.6 & 10.0 & 55.294 & 23.446 & $9.84\times10^{-4}$ & $4.13\times10^{-4}$ \\
$2\,\mathrm{exp}\bigl(2\,x_{1}^{2} + 2\,\sin(x_{3})\bigr) + 2$ & 17.2 & 8.4 & 39.844 & 19.168 & $9.55\times10^{-3}$ & $7.52\times10^{-3}$ \\
$32\,x_{3}^{6} + 2$ & 2.0 & 1.0 & 2.084 & 0.000 & $1.10\times10^{-2}$ & $2.93\times10^{-3}$ \\
\bottomrule
\end{tabular}}
\caption{Results of LLM-informed and original FEX runs on 50 randomly generated Poisson functions.}
\label{tab-poisson}
\end{table*}

\end{document}

%% file: ex_shared.tex
\usepackage{lipsum}
\usepackage{amsfonts}
\usepackage{graphicx}
\usepackage{epstopdf}
\usepackage{booktabs}
\usepackage{subcaption}
\usepackage{multirow,multicol}
\usepackage{algorithm} 
\usepackage{algpseudocode} 
\ifpdf
  \DeclareGraphicsExtensions{.eps,.pdf,.png,.jpg}
\else
  \DeclareGraphicsExtensions{.eps}
\fi

\newsiamremark{remark}{Remark}
\newsiamremark{hypothesis}{Hypothesis}
\crefname{hypothesis}{Hypothesis}{Hypotheses}
\newsiamthm{claim}{Claim}
\newsiamremark{fact}{Fact}
\crefname{fact}{Fact}{Facts}

\headers{Unraveling Symbolic Structures in PDEs with LLMs}{R. Bhatnagar, L. Liang, K. Patel and H. Yang}

\title{From Equations to Insights: Unraveling Symbolic Structures in PDEs with LLMs~\thanks{The first three authors contribute equally.}}

\author{Rohan Bhatnagar~\thanks{Department of Computer Science, University of Maryland, College Park, Maryland, USA 20742
  (\email{rbhatna1@terpmail.umd.edu}).} \and Ling Liang~\thanks{Department of Mathematics, The University of Tennessee, Knoxville, Tennessee, USA 37916 
  (\email{liang.ling@u.nus.edu}).}
\and Krish Patel~\thanks{Department of Computer Science, University of Maryland, College Park, Maryland, USA 20742
  (\email{kripatel@terpmail.umd.edu}).} \and Haizhao Yang~\thanks{(Corresponding author) Department of Mathematics, and Department of Computer Science, University of Maryland, College Park, Maryland, USA 20742
  (\email{hzyang@umd.edu}).}}

\usepackage{amsopn}
